\documentclass[num-refs]{wiley-article}

\usepackage{multirow}
\usepackage{mathrsfs}
\usepackage{amssymb}
\usepackage{subfigure}
\usepackage{amsmath}
\usepackage[ruled]{algorithm2e}
\usepackage{algorithmicx}
\usepackage{threeparttable}

\usepackage{siunitx}

\usepackage{color}
\newcommand{\mb}{\mathbf}
\newtheorem{thm}{\bf Theorem}
\newtheorem{defn}{Definition}

\renewcommand{\arraystretch}{0.5}

\newcommand{\our}{MeGuide}
\newcommand{\ripplewalk}{RWT}
\newcommand{\feature}{\textit{Feature Smoothness}}
\newcommand{\distance}{\textit{Connection Failure Distance}}
\newcommand{\gcn}{\text{GCN}}
\newcommand{\gat}{\text{GAT}}
\newcommand{\gnn}{\text{GNNs}}

\papertype{Original Article}
\paperfield{Journal Section}

\title{Measuring and Sampling: A Metric-guided Subgraph Learning Framework for Graph Neural Network}

\author[1\authfn{1}]{Jiyang Bai}
\author[2\authfn{1}]{Yuxiang Ren}
\author[3]{Jiawei Zhang}

\contrib[\authfn{1}]{Should be considered joint first author}

\affil[1]{Department of Computer Science, Florida State University, Tallahassee, Florida, 32306, USA. Email: bai@cs.fsu.edu}
\affil[2]{IFM Lab, Department of Computer Science, Florida State University, Tallahassee, Florida, 32306, USA. Email: yuxiang@ifmlab.org}

\affil[3]{IFM Lab, Department of Computer Science, University of California, Davis, Davis, California, 95616, USA. Email: jiawei@ifmlab.org}

\corraddress{Yuxiang Ren, IFM Lab, Department of Computer Science, Florida State University, Tallahassee, Florida, 32306, USA}
\corremail{yuxiang@ifmlab.org}

\fundinginfo{National Science Foundation, Grant Number: IIS-1763365}

\runningauthor{Bai \& Ren et al.}

\begin{document}

\begin{frontmatter}
\maketitle

\begin{abstract}
Graph neural network (GNN) has shown convincing performance in learning powerful node representations that preserve both node attributes and graph structural information. However, many GNNs encounter problems in effectiveness and efficiency when they are designed with a deeper network structure or handle large-sized graphs. Several sampling algorithms have been proposed for improving and accelerating the training of GNNs, yet they ignore understanding the source of GNN performance gain. The measurement of information within graph data can help the sampling algorithms to keep high-value information while removing redundant information and even noise. In this paper, we propose a \textbf{Me}tric-\textbf{Guide}d ({\our}) subgraph learning framework for GNNs. {\our} employs two novel metrics: {\feature} and {\distance} to guide the subgraph sampling and mini-batch based training. {\feature} is designed for analyzing the feature of nodes in order to retain the most valuable information, while {\distance} can measure the structural information to control the size of subgraphs. We demonstrate the effectiveness and efficiency of {\our} in training various GNNs on multiple datasets.

\keywords{Graph Neural Network; Subgraph Sampling; Training Optimization;}
\end{abstract}
\end{frontmatter}

\section{Introduction}
\label{sec:intro}
In recent years, graph neural networks (GNNs) have progressed in graph representation learning. Numerous real-world graph-related applications, such as social media~\cite{JZhang19}, fake news detection~\cite{ren2020hgat} and knowledge graphs~\cite{WMWG17}, exhibit the favorable property of graph neural networks. The core idea of GNN is to aggregate the feature information of the nodes' neighbors through neural networks to update node representations, which combine both the independent information of the nodes and corresponding structural information. As the scale of the graph increases and the GNN model architecture goes deeper, several challenging problems will emerge in learning GNNs: \textit{neighbors explosion}, \textit{node dependence}, and \textit{oversmoothing}~\cite{bai2020ripple}. GNNs learn high-level representations through a recursive neighbors aggregation scheme~\cite{XHLJ10}, which causes the number of neighbors to explode. This problem is described as \textit{neighbors explosion}~\cite{bai2020ripple}, which leads to exponentially-growing computation complexity. Because neighboring nodes are interdependent in the learning process, most current GNNs have to work based on the full graph. This limitation is the \textit{node dependence} as described in~\cite{chiang2019cluster}, which brings out serious memory and computation bottlenecks in handling large-sized graphs. At last, when GNNs go deeper and learn on the full graph, node representations from different clusters will be aggregated~\cite{zhao2019pairnorm}, which contradicts the smoothness assumptions (close nodes are similar). This \textit{oversmoothing} issue can lead to learned node representations indistinguishable~\cite{li2018deeper,rong2019dropedge}. 

Several sampling-based methods have been proposed to deal with the aforementioned problems. Among them, one direction is the node-sampling method~\cite{hamilton2017inductive, chen2018fastgcn, chen2017stochastic, zou2019layer, rong2019dropedge}. GraphSAGE~\cite{hamilton2017inductive} samples features from local node neighborhoods to learn a function that updates node representations. The neighborhoods sampling is able to mitigate the harassment from \textit{neighbors explosion} and \textit{node dependence}. DropEdge~\cite{rong2019dropedge} randomly removes some edges from the input graphs, which is a kind of node-sampling essentially, to overcome the \textit{oversmoothing}. However, for large-sized graphs, even if neighboring nodes are sampled, it is still difficult to avoid loading the full graph during training, which brings a lot of memory overhead. At the same time, the node-sampling methods focus on sampling distribution (e.g., importance sampling~\cite{zou2019layer}) to reduce the high-level approximation variance, but this kind of sampling cannot precisely quantify the information gain from each node. 

The other direction is the graph-sampling method~\cite{chiang2019cluster,zeng2019graphsaint, bai2020ripple}, which can deal with the memory overhead of loading the full graph. Cluster-GCN~\cite{chiang2019cluster} constructs the subgraphs mini-batch by clustering on the full graph to train the GCN. However, the clustering algorithm itself is greatly affected by the structure of the full graph, that is, the size of the clustered subgraph is uncontrollable. As a heuristic method, Cluster-GCN is difficult to guarantee the generalization performance of graph data with different structures. Unlike Cluster-GCN, GraphSAINT~\cite{zeng2019graphsaint} constructs mini-batches by sampling subgraphs with the support of several random samplers. RippleWalk~\cite{bai2020ripple} proposes a training framework together with the ripple walk sampler to consider randomness and connectivity of sampled subgraphs. Graph-sampling methods using subgraphs to train GNNs can handle the three problems mentioned above simultaneously to a certain extent but only focus on the training phase of GNN models. GraphSAINT and RippleWalk still need to load the full graph when using the trained model to make predictions, which challenges the memory space. What is more, current graph-sampling methods along with their samplers are limited in understanding the source of information gains and noise when GNNs are trained. In fact, the understanding of information gain and noise can effectively help graph-sampling methods to obtain higher quality subgraphs that are the key to the training performance.

In this paper, we propose a general learning framework, namely \textbf{Me}tric-\textbf{Guide}d ({\our}), for graph neural networks. {\our} is a mini-batch based learning framework, which can help GNNs learn on the sampled subgraphs instead of the full graph. The full GNN is trained and updated based on the mini-batch gradient. {\our} employs a novel Metric-Guided Sampler to sample subgraphs for the mini-batch, whose sampling logic is mainly based on the two novel metrics we introduce in this paper: {\feature} and {\distance}. In subgraph sampling, the node selection and the size of the subgraph are two important factors that determine the shape of the subgraph. For node selection, the existing methods~\cite{rong2019dropedge,zeng2019graphsaint} randomly sample and ignore the information provided by node feature vectors. {\feature} measures the information gain of the node feature vectors, which is used by Metric-Guided Sampler to select nodes with more information and drop nodes with redundant information during the subgraph sampling process. In addition, according to the basic assumption on the graph structure data, the closer the nodes are, the more similar and the more likely they are to have the same label. It's reasonable to consider that neighbors with different labels contribute to negative disturbance~\cite{hou2019measuring}. When the distance between nodes exceeds a certain metric, there is a greater probability for the sampled nodes to have different labels. This metric we define in this paper is {\distance}, which is utilized by {\our} Sampler to control the longest multi-hop connection in subgraphs. In this way, {\our} can help GNNs avoid unexpected aggregations (e.g., aggregating nodes with different labels), which is the primary cause of \textit{oversmoothing}~\cite{bai2020ripple}. To a certain extent, the size of subgraphs can also be determined, but in other sampling methods, it is a hyper-parameter. With the understanding of information gains and disturbance during sampling subgraphs, {\our} can help the training of GNN models to achieve better performance in both effectiveness and efficiency. In addition, for how to apply trained GNN models to perform prediction tasks based on subgraphs, {\our} proposes a representation aggregation-based scheme so that it is no longer necessary to load the full graph in the prediction phase.

The contributions of our work are summarized as follows:
\begin{itemize}
	\item We define two metrics {\feature} and {\distance} based on the smoothness of node features and the connectivity of graphs respectively to measure the quantity and quality of information gain between nodes.  
	
	\item We propose a general learning framework {\our} for different GNN models. {\our} samples high-value subgraphs guided by {\feature} and {\distance} for training GNN models. 
	
	\item For the case of the memory bottleneck when using learned GNN models to make predictions on a single large graph, {\our} employs the representation aggregation-based prediction on subgraphs, which is an unconsidered problem left by existing subgraph-based training methods.
	
	\item We conduct extensive experiments on 5 benchmark graph datasets with different sizes to demonstrate both the effectiveness and efficiency of {\our}. The results show the superiority of {\our} subject to the training efficiency and the performance of prediction.
\end{itemize}

The remaining paper is organized as follows. We review the related works in Section~\ref{sec:relate}. Then we introduce preliminaries and background in Section~\ref{sec:preliminaries}. The proposed framework {\our} is introduced in Section~\ref{sec:solution}, whose effectiveness is evaluated in Section~\ref{sec:experiment}. Finally, we conclude this paper in Section~\ref{sec:conclusion}.

\section{Related Work}
\label{sec:relate}

\subsection{Graph neural network}
Graph neural networks (GNN) aim at the machine learning tasks involving graph-structured data. Bruna et al.~\cite{bruna2013spectral} express the idea of graph information construction based on the theorem of the spectrum of the graph Laplacian. Later, spatial-based GNNs~\cite{gat,monti2017geometric} define graph convolutions directly based on a node's spatial relations. Different from the spectral-based {\gnn}, where the weights of edges in the graph have been determined before the training, the connections (edges) between nodes and the weights of connections can be automatically learned in the training process. For example, \cite{gat} applies the attention mechanism to learn the attention weights for edges in each training epoch. Nonetheless, both the spectral-based and spatial-based {\gnn} can be regarded as an information propagation-aggregation mechanism, and such mechanism is achieved by the connections and multi-layers structure. 
As a popular research topic, there are already many graph neural networks ~\cite{kipf2016semi,hamilton2017inductive,xinyi2018capsule, sun2019infograph,ying2018hierarchical,ren2021label} having shown awe-inspiring capabilities in handling graph structure data. More progress on graph neural networks can be referred to the surveys~\cite{zhou2020graph,wu2020comprehensive}.

\subsection{Optimization in GNN}
Many GNN models are limited by the problems from three aspects: \textit{node dependence}, \textit{neighbors explosion}, and \textit{oversmoothing}. In response to these problems, some related work has been proposed from different directions.
\subsubsection{Node dependence}
\textit{Node dependence}~\cite{chiang2019cluster} forces GNNs to be trained on the entire graph, which leads to the slow training process. More specifically, the information aggregation involves the adjacency matrix of the full graph in each training epoch. To deal with such problem, The works~\cite{chiang2019cluster, zeng2019graphsaint} apply the concept of subgraph training methods. The essence of subgraph training is to collect a batch of subgraphs from the full graph and use them during the training process. The subgraph collecting strategies can be various within different methods. Chiang et al.~\cite{chiang2019cluster} divide the full graph into subgraphs according to the clustering results. Training the GCN by clustered subgraphs can avoid unexpected aggregations from different clusters, but the size of clustered subgraphs is uncontrollable. More importantly, the subgraphs sampled by clustering share no joint node, thus the connection information within the full graph will be partly discarded. Another work in \cite{zeng2019graphsaint} applies several subgraph sampling ideas (e.g., on the node, edge, random walk) when training the GCN for inductive tasks. There are also other approaches to optimize the {\gnn} frameworks. \cite{defferrard2016convolutional,kipf2016semi,levie2018cayleynets,liao2019lanczosnet} optimize the localized filter in order to reduce the time cost of training on the full graph. Further, \cite{henaff2015deep,li2018adaptive} reduce the number of learnable parameters by dimensionality reduction and residual graph laplacian respectively. But these approaches do not alleviate the space complexity problem. 

\subsubsection{Neighbors expansion}
\textit{Neighbors expansion} makes deep GNNs difficult to being implemented. Because learning a single node requires embeddings from its neighbors, and the quantity may be explosive when a GNN goes deeper. 
Some research works deal with \textit{neighbors explosion} by neighbors sampling~\cite{hamilton2017inductive,chen2018fastgcn,ying2018graph,chen2017stochastic}. GraphSAGE~\cite{hamilton2017inductive} proposes to sample neighbors when aggregating information for every node. FastGCN~\cite{chen2018fastgcn} regards the neighbors following specific distribution, and then does neighbor sampling on the distribution level. VR-GCN~\cite{chen2017stochastic} conducts the neighbors sampling with variance reduction for each node. In other directions, several models~\cite{gao2018large,xu2018representation,lee2019graph} select specific neighbors based on defined metrics to avoid the explosive quantity. In~\cite{klicpera2019predict,haonan2019graph,abu2019mixhop,chen2020scalable}, the propagation-aggregation mechanism is optimized to enable the node to capture long-distance information even with a relatively shallow structure. \cite{rong2019dropedge} randomly remove edges from input graphs to handle the \textit{neighbor explosion}. The above related works all focus on the neighbor-level sampling but still have the same space complexity with original {\gnn}.

\subsubsection{Oversmoothing}
The problem of \textit{oversmoothing} in the GCN was introduced in~\cite{li2018deeper}. When GNNs go deep, the performance suffers from \textit{oversmoothing} where node representations from different clusters become mixed up~\cite{zhao2019pairnorm}. The node information propagation-aggregation mechanism can be regarded as one type of random walk within the graph. With the increasing of walking steps, the representations of nodes will finally converge to stable status. Such convergence would impede the performance of {\gnn} and make the nodes indistinguishable in the downstream tasks. Some related works have been proposed to deal with the \textit{oversmmothing}. \cite{gresnet} comes up with the suspended animation and utilizes the residual networks to mine the advantages of deeper networks. In such a case, the depth of GNN models can reach more than fifty with a better performance. 

Till now, the optimized methods based on subgraph learning are limited~\cite{zeng2019graphsaint,chiang2019cluster,bai2020ripple}. Their measurement of the quality of sampled subgraphs is not comprehensive enough, and the sampling of subgraphs is precisely the key to this type of method. In this paper, we propose a subgraph learning framework {\our} for GNNs, which employs two novel metrics: {\feature} and {\distance} to guide the subgraph sampling and mini-batch based training. High-quality subgraphs can ultimately help {\our} to better empower different GNNs learning.

\section{Preliminaries and Background}
\label{sec:preliminaries}
In this section, we first introduce the preliminaries about general GNN models. Then we elaborate the basic idea of subgraph-based training for GNN models.
\subsection{General GNN models}
We denote a graph as $\mathcal{G} = (\mathcal{V}, \mathcal{E})$, where $\mathcal{V}$ and $\mathcal{E}$ represent the set of nodes and edges of $\mathcal{G}$ respectively. Most widely used GNN mdoels (e,g, GAT~\cite{gat}, GCN~\cite{kipf2016semi}, GIN~\cite{XHLJ10}) follow the recursive neighbors aggregation scheme to update the representation for each node. For node $v_i$, we use $\mathcal{N}_{v_i} = \{v_j: e_{v_i, v_j}\in\mathcal{E}\}$ to represent the set of its neighbors, where $e_{v_i, v_j}$ denotes the edge between $v_i$ and $v_j$. Each node has an initial feature vector $x_v \in \mathbb{R}^d$ with dimension $d$. The initial feature matrix $\mathbf{X} \in \mathbb{R}^{|\mathcal{V}|\times d}$ is consitute of all nodes' initial feature vectors. The neighbors aggregation scheme of GNN models can be represented generally as:   
\begin{equation}\label{equ:aggre}
	\begin{aligned}
		&\mb{H}^{(0)} = \mb{X}\\
		&\mb{H}^{(l+1)}_{v_i} = \sigma(\sum_{j\in {\mathcal N}_{v_i}} \mb{\alpha}_{ij}\cdot \mb{H}^{(l)}_{v_j}\mb{W}^{(l)})
	\end{aligned}
\end{equation}
Here, $\mb{H}^{(0)}$ is the input feature matrix of a GNN model, and the $\mb{H}^{(l)}_{v_j}$ is the hidden representation of node $v_j$ in the $l_{th}$ layer; $\sigma$ is an activation function; $\mb{W}^{(l)}$ is the learnable parameter matrix for linear transformation; $\mb{\alpha}$ is a variant of cofficient matrix, which has different definitions according to different GNN models. For example, in GCN~\cite{kipf2016semi}, $\mb{\alpha} = \widetilde{\mb{A}}$ is the normalized adjacency matrix. When the GNN model is GAT~\cite{gat}, $\mb{\alpha}$ is the attention weights matrix learned in current round. Through the feedforward layer computing, the hidden representation of node $v_i$ is updated by aggregating its current representation and neighbors' hidden representations. With the support of a mapping function (e.g. a fully-connected layer), the learned representation can serve for downstream tasks such as node classification.

\subsection{Subgraph-based training for GNN models}
The learning scheme shown in Equation~\ref{equ:aggre} needs to take the full graph $\mathcal{G}$ as input. Training GNN models with the full graph, especially facing large-sized graphs, may easily lead to aforementioned three problems:\textit{neighbors explosion}, \textit{node dependence}, and \textit{oversmoothing}~\cite{bai2020ripple}. To solve these problems, subgraph-based training methods~\cite{bai2020ripple,zeng2019graphsaint} employ subgraphs of $\mathcal{G}$ to construct the mini-batch in each training iteration and update the complete GNN models based on the mini-batch gradient. We use the $\mathcal{G}_t = (\mathcal{V}_t, \mathcal{E}_t)$ to denote a subgraph of $\mathcal{G}$, where $\mathcal{V}_t\subseteq \mathcal{V}$ and $\mathcal{E}_t \subseteq \mathcal{E}$; In this way, the neighbor aggregation procedure of GNN models when training with the subgraph $\mathcal{G}_t$ can be represented as:
\begin{equation}\label{equ:sub_aggre}
	\begin{aligned}
		&\mb{H}^{(0)} = \mb{X}_{\mathcal{G}_t}\\
		&\mb{H}^{(l+1)}_{v_i} = \sigma(\sum_{j\in \mathcal{N}^{\mathcal{G}_t}_{v_i}} \mb{\alpha}^{\mathcal{G}_t}_{ij}\cdot \mb{H}^{(l)}_{v_j}\mb{W}^{(l)})
	\end{aligned}
\end{equation}

Here, $\mathcal{N}^{\mathcal{G}_t}_{v_i}$ is the neighbor nodes set of node $v_i$ in $\mathcal{G}_t$; $\mb{\alpha}^{\mathcal{G}_t}$ corresponds to the cofficient matrix of $\mathcal{G}_t$.
The aggregated representations of nodes in $\mathcal{G}_t$ are used to calculate the loss and gradients in order to update the complete GNN models.

However, different from previous data types such as the image using mini-batch gradient descent, the nodes in a graph are not independent from each other. In this way, the subgraph is equivalent to dropping some edges, which may lead to losing dependency information (connections). For this concern, RippleWalk~\cite{bai2020ripple} provides a theoretical analysis in Theorem 1 and 2, which prove that the subgraph-based mini-batch gradient descent is still reliable for optimizing GNN models. However, the quality of subgraphs will have a great impact on the training performance, and this is also the problem we try to solve in the paper: sampling more effective subgraphs for the learning of GNN models.

\section{Proposed Methodology}\label{sec:solution}
In this section, we first introduce two metrics {\feature} and {\distance} that are the keys to sampling effective subgraphs. Then we propose the metric-guided sampling method, and follow it up by presenting the subgraph-based training and representation aggregation-based prediction of the proposed {\our} framework. 

\subsection{Subgraph Sampling Metrics}\label{sec:metrics}
The neighbors aggregation scheme works to collect and aggregate neighboring information to update node representations. The neighboring relationship on a graph can indicate the closeness among nodes to a certain extent. We analyze the neighboring information aggregation process from the perspective of information gain~\cite{kullback1951information}. When neighboring nodes $\mathcal{N}_{v_i}$ and $v_i$ are too similar, they cannot bring much information gain to the final aggregated representation of $v_i$. At the same time, the neighbors aggregation scheme will aggregate the multi-hop neighboring nodes with more stacked
convolution layers. When the hop distance is too long, node representations from different clusters become mixed up~\cite{zhao2019pairnorm}, which is the primary reason of \textit{oversmoothing}. Therefore, the hop distance should be related to the effectiveness of information gain during the neighbors aggregation process.

Based on the above two analyses, we define {\feature} and {\distance} to measure the quantity and quality of information gain in the aggregation process, respectively. These two metrics can guide the sampling method to obtain subgraphs that can bring high-quality information gain for the learning of GNN models.

\subsubsection{Feature Smoothness}\label{sec:featuresmooth}   
In order to quantify the information obtained from neighboring nodes, we use Kullback-Leibler divergence to measure the information gain between two connected nodes.


\begin{defn}\label{def:info_gain}
	(Information Gain between Connected Nodes): Given a graph $\mathcal{G} = (\mathcal{V}, \mathcal{E})$, for each node $v\in \mathcal{V}$ and its representation follows the distribution $Q$; the node $v'\in \mathcal{N}_{v}$, and the aggregated representation of $v$ from $v'$ (denoted as $AGG(v, v')$), follows the distribution $Q_{AGG}$. Assume $Q$ and $Q_{AGG}$ are over the same feature space $\mathcal{X}$. The information gain of the graph $\mathcal{G}$ can be measured by the Kullback-Leibler divergence~\cite{kullback1951information} as:
	\begin{equation}
		D_{KL}(Q_{AGG}||Q) = \int_{\mathcal{X}} Q_{AGG}(x)\cdot \log\frac{Q_{AGG}(x)}{Q(x)} dx
	\end{equation}
	Similarly, for a specific node $v_i\in\mathcal{V}$ and its neighboring node $v_j\in \mathcal{N}_{v_i}$, we can denote the information gain of $v_i$ from $v_j$ as $D_{KL}(Q_{AGG(v_i, v_j)}||Q_{v_i})$.
	
\end{defn}

From Hou et al.~\cite{hou2019measuring}, the Kullback-Leibler divergence can measure the information gain between one node and all neighboring nodes. Since for node $v$, if the features of its neighboring nodes are equal to $v$'s feature, obviously the divergence is $0$. On the other hand, when the neighboring nodes possess significantly different feature distributions compared with the central node, the divergence is relatively large.
However, in practice, the accurate distributions of node features are unknown. Therefore we propose the following {\feature} to quantify the actual information gain.

\begin{defn}\label{def:feat_smooth}
	(Feature Smoothness): Given a graph $\mathcal{G} = (\mathcal{V}, \mathcal{E})$, the feature smoothness of the graph is defined as:
	\begin{equation}
		\lambda_f = \frac{||\sum_{v\in\mathcal{V}}\left(\sum_{v'\in\mathcal{N}_v}(x_v - x_{v'})\right)^2||_1}{|\mathcal{E}|\cdot d}
	\end{equation}
	
	where $||\cdot||_1$ is the Manhattan norm, $x_{v}\in \mathbb{R}^d$ and $x_{v'}\in \mathbb{R}^d$ are the initial features of $v$ and $v'$, respectively. Similarly, we can also define the feature smoothness between connected nodes $v_i$ and $v_j$ as: 
	
	\begin{equation}
		\lambda_{f^{(v_i, v_j)}} = \frac{||(x_{v_i} - x_{v_j})^2||_1}{d}
	\end{equation}

\end{defn}
The $\lambda_f$ in Definition~\ref{def:feat_smooth} counts the sum of norm-2 distance among connected nodes, which can provide an overall feature divergence of the entire graph. A higher $\lambda_f$ indicates that the feature signals of a graph have \textit{higher frequency}~\cite{hou2019measuring}, meaning that the connected nodes in the graph are more likely dissimilar. While $\lambda_f$ is an overall metric of the entire graph, the $\lambda_{f^{(v_i, v_j)}}$ measures the similarity between connected nodes $v_i$ and $v_j$. Different from the Kullback-Leibler divergence that is unknown in practice, instead $\lambda_{f^{(v_i, v_j)}}$ can be easily calculated for specific connected nodes. Therefore we propose to explicitly quantify the information gain with the help of $\lambda_{f^{(v_i, v_j)}}$, and state the relation between information gain and {\feature} in the following theorem.

\begin{thm}\label{thm:positive_relate}
	Given a $\mathcal{G} = (\mathcal{V}, \mathcal{E})$, the information gain of node $v_i$ from its neighboring node $v_j$ in Definition~\ref{def:info_gain} is positively related to their feature smoothness $\lambda_{f^(v_i, v_j)}$, i.e., 
	\begin{equation}
		D_{KL}(Q_{AGG(v_i, v_j)}||Q_{v_i}) \sim \lambda_{f^(v_i, v_j)}
	\end{equation}

\end{thm}

Before proving Theorem~\ref{thm:positive_relate}, here we provide a lemma to assist the proof.

\begin{lemma}\label{lem:1}
	Assume $Q$ is the distribution of node $v\in\mathcal{V}$ and $S$ is the distribution of the neighboring nodes$\sum_{v'\in\mathcal{N}_v}v'$, the Kullback-Leibler divergence $D_{KL}(S||Q)$ is positively related to $Var(|\mathcal{N}_v| \cdot x_v - \sum_{v'\in \mathcal{N}_v}x_{v'})$, i.e., 
	\begin{equation}
		D_{KL}(S||Q) \sim Var(|\mathcal{N}_v| \cdot x_v - \sum_{v'\in \mathcal{N}_v}x_{v'})
	\end{equation}
	where $Var(\cdot)$ denotes the variance, $|\mathcal{N}_v|$ is the number of nodes in $\mathcal{N}_v$.
\end{lemma}

\begin{proof}{\textbf{of \ Lemma 1\ \ \ \ \ }}
	For $D_{KL} (S||Q)$, since the explicit formulas of $S$ and $Q$ are unknown, we use the discrete space approach: the histogram, to estimate $S$ and $Q$. In detail, we divide the feature space $\mathcal{X} = [0,1]^d$ into $r^d$ bins $\{H_1, H_2, \dots H_{r^d}\}$ evenly, and the length is $\frac{1}{r}$ and dimension is $d$. Following the distributions of $S$ and $Q$, there are $2|\mathcal{E}|$ corresponding samples of nodes (each connected two nodes can be regarded as the central node and the neighboring node, and vice versa) in total, we use the $|H_i|_Q$ and $|H_i|_S$ to denote the number of samples falling into bin $H_i$. In this way, we have
	\begin{equation}
		\begin{split}
			D_{KL} (S || Q) \simeq& \sum_{i=1}^{r^d} \frac{|H_i|_S}{2 |\mathcal{E}|}\cdot \log \frac{\frac{|H_i|_S}{2|\mathcal{E}|}}{\frac{|H_i|_Q}{2|\mathcal{E}|}}\\
			=& \frac{1}{2|\mathcal{E}|}\sum_{i=1}^{r^d}|H_i|_S\cdot\log \frac{|H_i|_S}{|H_i|_Q}\\
			=&  \frac{1}{2|\mathcal{E}|} (\sum_{i=1}^{r^d} |H_i|_S\cdot \log |H_i|_S -\sum_{i=1}^{r^d} |H_i|_S\cdot \log |H_i|_Q )\\
			=& \frac{1}{2|\mathcal{E}|} (\sum_{i=1}^{r^d}  |H_i|_S\cdot \log |H_i|_S - \sum_{i=1}^{r^d} |H_i|_S\cdot \log ( |H_i|_S + \Delta_i))\\
		\end{split}
	\end{equation}
	where $\Delta_i = |H_i|_Q - |H_i|_S$. In this way, we can expand the term $\sum_{i=1}^{r^d} |H_i|_S\cdot \log ( |H_i|_S + \Delta_i)$ by applying the second-order Taylor approximation at the point 0 as 
	\begin{equation}
		\sum_{i=1}^{r^d}|H_i|_S \cdot \log ( |H_i|_S + \Delta_i)
		\simeq \sum_{i = 0}^{r^d} |H_i|_S ( \log  |H_i|_S + \frac{\ln2}{ |H_i|_S}\cdot \Delta_i - \frac{\ln 2}{2( |H_i|_S)^2}\cdot\Delta_i^2)
	\end{equation}
	Note that the number of samples from  $S$ and $Q$ are the same, which means
	
	\begin{equation}
		\sum_{i = 0}^{r^d} |H_i|_S = \sum_{i = 0}^{r^d} |H_i|_Q = 2|\mathcal{E}|
	\end{equation}
	So we have $\sum_{i = 0}^{r^d} \Delta_i = 0$. Therefore, 
	\begin{equation}
		\begin{split}
			D_{KL} (S || Q) \simeq& \frac{1}{2|\mathcal{E}|} (\sum_{i=1}^{r^d}  |H_i|_S\cdot \log |H_i|_S - \sum_{i=1}^{r^d} |H_i|_S\cdot \log ( |H_i|_S + \Delta_i))\\
			\simeq& \frac{1}{2|\mathcal{E}|} (\sum_{i=1}^{r^d}  |H_i|_S\cdot \log |H_i|_S - \sum_{i = 0}^{r^d} |H_i|_S ( \log  |H_i|_S + \frac{\ln2}{ |H_i|_S}\cdot \Delta_i - \frac{\ln 2}{2( |H_i|_S)^2}\cdot\Delta_i^2))\\
			=& \frac{1}{2|\mathcal{E}|} \sum_{i=1}^{r^d}(\frac{\ln 2}{2|H_i|_S}\Delta_i^2 - \ln 2\Delta_i)\\
			=& \frac{\ln 2}{4|\mathcal{E}}| \sum_{i=1}^{r^d}\frac{\Delta_i^2}{|H_i|_S}
		\end{split}
	\end{equation}
	For $S$ and $Q$, we consider the samples of $S$ as $\{x_v : v\in\mathcal{V}\}$ and samples of $Q$ as $\{\frac{1}{|\mathcal{N}_v|}\sum_{v'\in\mathcal{N}_v} x_{v'} : v\in\mathcal{V}\}$ with counts $|\mathcal{N}_v|$ for node $v$. In this way, the difference between $S$ and $Q$ can be represented by the expectation of  $|\mathcal{N}_v| x_v - \sum_{v'\in\mathcal{N}_v} x_{v'}$. Meanwhile, the discrepancy of numbers of samples in each bin, i.e. $\Delta_i$, is positively correlated with the difference between $S$ and $Q$. If we regard the $|H_i|_S$ as constant, we can infer
	
	\begin{equation}
		\begin{split}
			D_{KL} (S || Q) \simeq&\frac{\ln 2}{4|\mathcal{E}}| \sum_{i=1}^{r^d}\frac{\Delta_i^2}{|H_i|_S}\\
			\simeq & \mathbb{E} [(|\mathcal{N}_v| x_v - \sum_{v'\in\mathcal{N}_v} x_{v'})^2]\\
			=& Var(|\mathcal{N}_v| x_v - \sum_{v'\in\mathcal{N}_v} x_{v'})
		\end{split}
	\end{equation}
	
\end{proof}


Based on the Lemma~\ref{lem:1}, state the proof of Theorem~\ref{thm:positive_relate}.
\begin{proof}{\textbf{of \ Theorem 1\ \ \ \ \ }}
	According to Lemma~\ref{lem:1}, for node $v_i$, $S$ is the distribution of its neighboring nodes $\sum_{v'\in\mathcal{N}_{v_i}} x_{v'}$. Considering one neighboring node of $v_i$ (e.g., $v_j\in\mathcal{N}_{v_i}$), the $D_{KL}(Q||S)$ in Lemma~\ref{lem:1} will derive into $D_{KL}(Q_{v_i}||Q_{v_j})$, and it has
	\begin{equation}
		\begin{split}
			D_{KL}(Q_{v_j}||Q_{v_i})&\sim Var(|\{v_j\}|\cdot x_{v_i} - \sum_{v'\in \{v_j\}}x_{v'})\\
			&= Var(x_{v_i} - x_{v_j})
		\end{split}
	\end{equation}
	According to the proof of Theorem 4 in~\cite{hou2019measuring} and our Definition~\ref{def:feat_smooth},
	\begin{equation}
		\begin{split}
			\lambda_{f^{(v_i, v_j)}} =& \frac{||(x_{v_i} - x_{v_j})^2||_1}{d}\\
			= & \frac{||\sum_{v\in\{v_i\}}(\sum_{v'\in\{v_j\}}x_v - x_{v'})^2||_1}{|\{v_i\}|\cdot d}\\
			= & \frac{||Var(x_{v_i} - x_{v_j})||_1}{d}\\
			\sim & Var(x_{v_i} - x_{v_j})
		\end{split}
	\end{equation}
	Therefore, we can find that 
	\begin{equation}
		D_{KL}(Q_{v_j}||Q_{v_i})\sim 	\lambda_{f^{(v_i, v_j)}}
	\end{equation}
	%
	Finally, since the $AGG_{(v_i, v_j)}$ in Definition~\ref{def:info_gain} is equivalent to weighted summation of $x_{v_i} $ and $x_{v_j}$ i.e., $x_{v_i} + \alpha_{j}x_{v_j}$, where $\alpha_{j}\in(0, 1)$ is an fixed value, and $D_{KL}(Q_{v_i}||Q_{v_i}) = 0$, thus we can conclude
	\begin{equation}
		\begin{split}
			D_{KL}(Q_{AGG(v_i, v_j)}||Q_{v_i}) \sim D_{KL}(Q_{v_j}||Q_{v_i})
			\sim \lambda_{f^{(v_i, v_j)}}
		\end{split}
	\end{equation}
\end{proof}

From Theorem~\ref{thm:positive_relate}, we can conclude that a higher $\lambda_{f^{(v_i, v_j)}}$ represents higher information between connected nodes. In such case, for a specific node $v_i$ we are able to measure its information gain from the neighboring node $v_j$ by computing the feature smoothness $\lambda_{f^{(v_i, v_j)}}$. In Section~\ref{subsec:sampler}, $\lambda_{f^{(v_i, v_j)}}$ serves as the metric to determine the neighboring nodes sampling strategy, where the neighboring nodes can provide more information gain will be sampled.

\subsubsection{Connection Failure Distance} 
As we mentioned before, a long hop distance can lead to unexpected aggregation between nodes from different clusters. When measuring the aggregation of two nodes from different clusters by feature smoothness, the information gain may be relatively large, but the quality of this information gain is not high. Considering the node classification task, nodes from two different clusters are likely to have different labels. Aggregating the representations of nodes with different labels essentially introduces negative information to damage the discrimination of node representation, which results in \textit{oversmoothing} finally. Therefore, in addition to measuring the quantity of information gain, we also need to explore the way to measure its quality.

\begin{figure}[t]
	\centering
	\begin{minipage}[l]{0.45\columnwidth}
		\centering
		\includegraphics[width=1\textwidth]{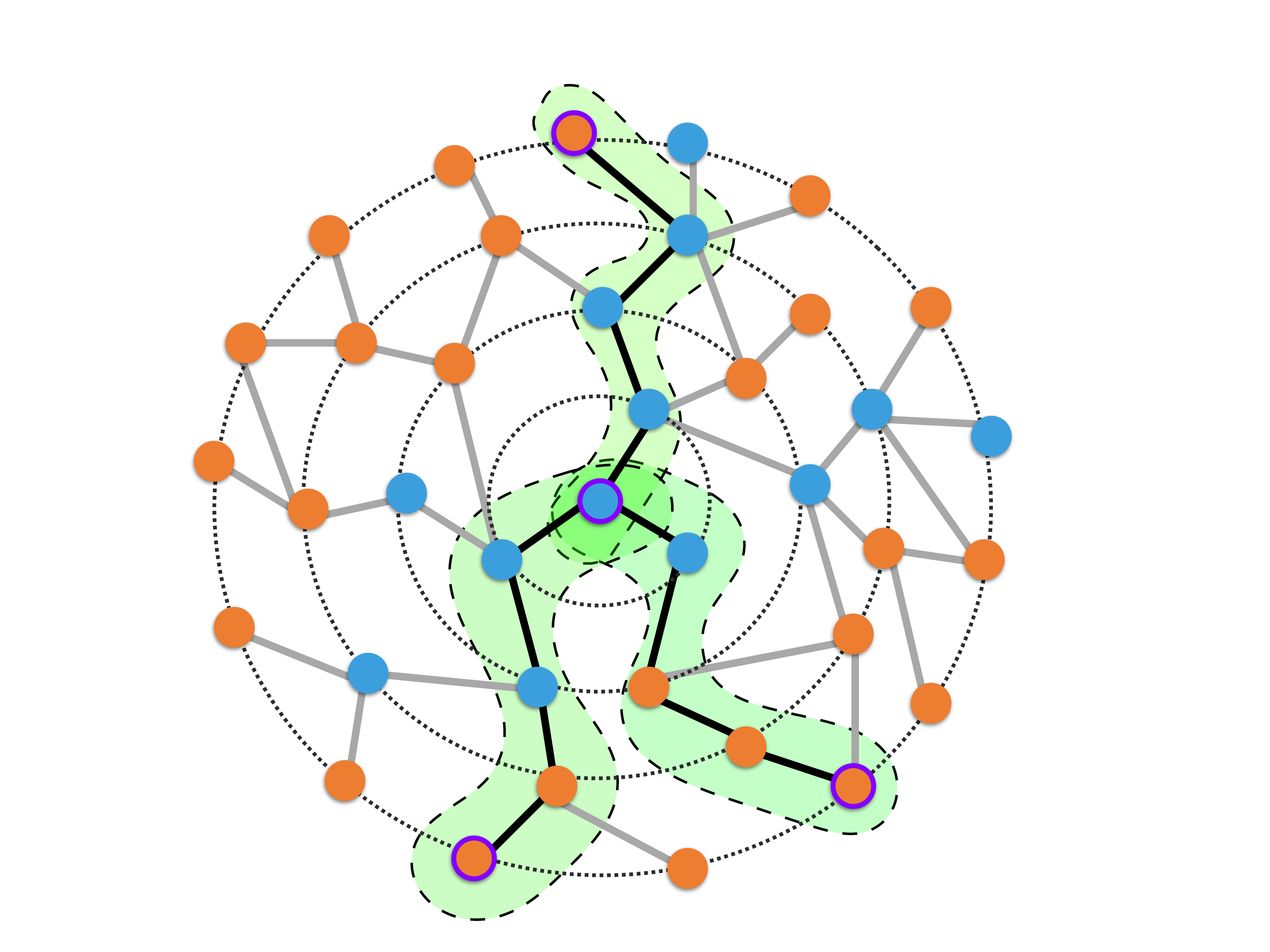}
	\end{minipage}
	\caption{An illustration of the failure connection. The different colors of the nodes represent different node labels. The three connections marked in green are failed connections.}
	\label{fig:connect_fail}
\end{figure}
Many previous works have analyzed through experiments that GNN models with deep stacking layers (long hop distances) bring very limited improvement in model performance~\cite{rong2019dropedge, zhang2019gresnet}, and the model performance even declines significantly~\cite{huang2020tackling}. This actually shows that the longer the hop distance, the more negative information is introduced, and the quantity of negative information will eventually exceed the positive information. It is reasonable to consider that neighbors with the same label contribute positive information to the information gain and other neighbors contribute negative disturbance. As the hop distance between two nodes becomes longer, the probability of two nodes having different labels rises, and the possibility of introducing negative information is also rising. We treat multi-hop connections between nodes with different labels as failure connections because these connections bring more negative information gain. We provide a simple illustration of the failure connection in Figure~\ref{fig:connect_fail}. The different colors of the nodes in the graph represent different classes. The three connections marked in green are failed connections since they link nodes with different labels. The longer the hop distance is, the probability that the central node encounters nodes with different labels increases. In this simple illustration, when the hop distance is greater than 3, the number of nodes with different labels will be far more than nodes with the same label. A large number of failure connections will make the negative information gain more than the positive one. Here we define a metric related to hop distance, namely {\distance}, to measure the quality of information gain.


\begin{defn}
	(Connnection Failure Distance): To measure the quality of the information gain brought by the neighbors aggregation process, we define the metric {\distance} as:
	\begin{equation}
		\lambda_d = \frac{\sum_{v_i \in \mathcal{V}}\max\{hop\_dis(v_i, v_j)|\ \  v_j \in \mathcal{V},\mathbb{I}(v_i, v_j) = 1\}}{|\mathcal{V}|} 
	\end{equation}   
\end{defn}
where $\mathbb{I}(v_i, v_j)$ is an indicator function: if the label $y_{v_i}$ of the node $v_i$ is the same as $y_{v_j}$ of the node $v_j$, it returns $1$, otherwise returns $0$. $hop\_dis(v_i, v_j)$ is the smallest number of hops between $v_i$ and $v_j$. If two nodes can not be connected through multi-hops, we set the value as $0$. $\lambda_d$ reflects the average longest hop distance between two nodes with the same label in $\mathcal{G}$. Neighbors aggregation longer than $\lambda_d$ is likely to be performed between two nodes with different labels, which contributes more negative information. Because of this, we name $\lambda_d$ as {\distance} to indicate that the multi-hop connections longer than $\lambda_d$ are meaningless.

\begin{thm}\label{thm:fail_dist}
	Given two nodes $v_i$ and $v_j$, if $hop\_dis(v_i, v_j) > \lambda_d$, $p(\mathbb{I}(v_i, v_j)=1) < \frac{1}{\#classes}$. Here, $p(\cdot)$ represents the probability and $\#classes$ denotes the total number of classes of nodes.
\end{thm}

\begin{proof}
	According to the basic assumption in the graph structure data, the closer the nodes are, the more similar and the more likely they are to have the same label.  From such assumption, given a $v_i$ we can derive that $p(\mathbb{I}(v_i, v_j) = 1)>\frac{1}{\#classes}$ can always be satisfied when $hop\_dis(v_i, v_j)$ is small enough. It is obvious since nodes closer to $v_i$ have the higher probability owning same label with $v_i$ compared to other nodes, and such higher probability is greater than $\frac{1}{\#classes}$. Thus we can always find a $d_0$ that satisfies $d_0 < \lambda_d$ and $p(\mathbb{I}(v_i, v_j) = 1 | hop\_dis(v_i, v_j)<= d_0) > \frac{1}{\#classes}$. In such case, $p(\mathbb{I}(v_i, v_j) = 1 | hop\_dis(v_i, v_j)>d_0) < \frac{1}{\#classes}$ and then we can get $p(\mathbb{I}(v_i, v_j) = 1 | hop\_dis(v_i, v_j)>\lambda_d) < \frac{1}{\#classes}$.
\end{proof}

According to Theorem~\ref{thm:fail_dist}, given a specific node $v_i$, any node that has longer hop distance than $\lambda_d$ is more likely to own different labels. Thus we can leverage $\lambda_d$ as the metric to guide the subgraph sampling process, aiming to avoid the negative information aggregation within the sampled subgraphs.

However, if we use $\lambda_d$ to measure the quality of the information gain, node labels are required for calculation. Since we cannot obtain the labels of all nodes for training, we utilize labeled nodes in the training set to estimate $\lambda_d$ instead. 

\begin{figure*}[t]
	\centering
	\begin{minipage}[l]{1.02\columnwidth}
		\centering
		\includegraphics[width=1\textwidth]{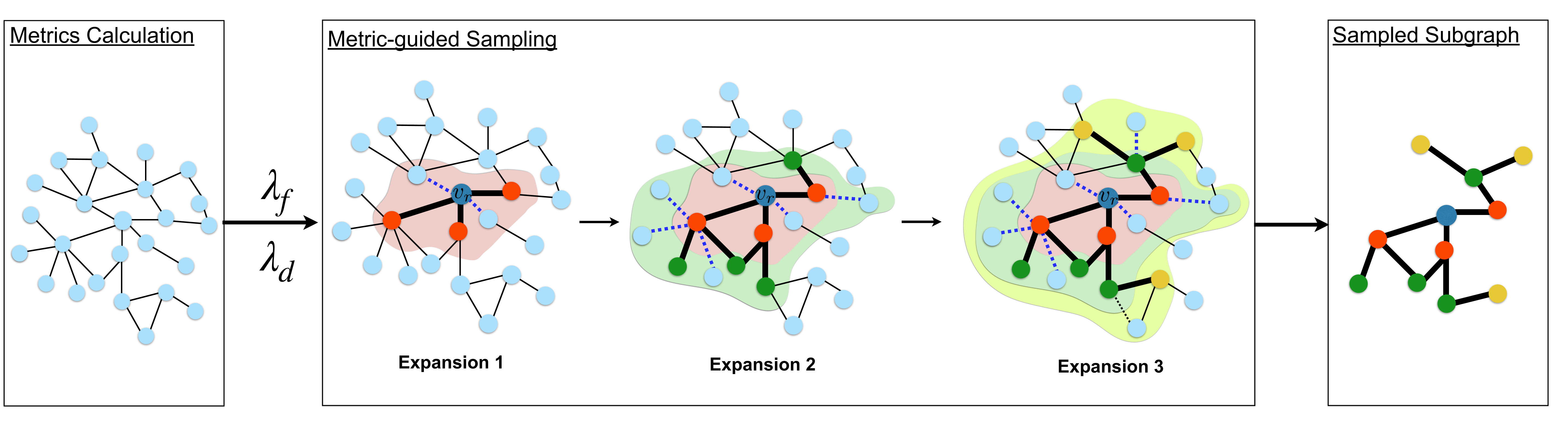}
	\end{minipage}
	\caption{An illustration of the metric-guided sampling process of {\our} sampler. In this example, $\lfloor \lambda_d/2 \rfloor=3$ determines the expansion step is $3$. The sampling process starts from one root node $v_s$. In each expansion, the background color region represents the neighbor set. For two connected nodes $v_i$ and $v_j$, if $\lambda_{f^{(v_i, v_j)}} \ge	\rho \lambda_f$, the edge between them is a bold black line and the node will be sampled. Otherwise, the edge is a blue dotted line, and the node will be dropped. After 3 expansion steps, the colored nodes represent the truly sampled nodes and constitute a subgraph. (Best viewed in color)}
	\label{fig:metric_sampling}
\end{figure*}

\subsection{Metric-Guided Subgraph Sampling}\label{subsec:sampler}

As a subgraph learning framework for GNN models, the subgraph sampling method is the key to the framework's performance. A suitable and accurate selection of neighboring nodes can retain node content information and graph topology at the same time, which can help achieve the same object as the full graph in Equation~\ref{equ:sub_aggre}. Based on the proposed two metrics in Section~\ref{sec:metrics}, we design the metric-guided subgraph sampling method {\our} Sampler to obtain effective subgraphs.

\begin{algorithm}[t]
	\caption{{\our} Sampler} 
	\KwIn{Target graph $\mathcal{G} = (\mathcal{V}, \mathcal{E})$; Connection Failure Distance $\lambda_d$; Graph Feature Smoothness $\lambda_f$; Feature Smoothness Hyper-parameter $\rho$}
	\KwOut{Subgraph $\mathcal{G}_k$}
	Initiate ${\mathcal{G}}_k = (\mathcal{V}_k, \mathcal{E}_k)$ with $\mathcal{V}_k = \emptyset$; Set expansion step $m = 0$; Set current expansion nodes set $CS = \emptyset$; Set next expansion nodes set $NS = \emptyset$;\\
	Randomly pick the root node $v_{r}$, add $v_{r}$ into $\mathcal{V}_k$ and $CS$;
	\While {$m \le \lfloor \lambda_d/2 \rfloor$}
	{
		\For {$v_i \in CS$}
		{
			\For {$v_j \in \{v_j| e_{v_i, v_j}\!\in\!\mathcal{E}, v_j\!\in\!\mathcal{V}\setminus\mathcal{V}_k\}$}
			{	
				\uIf {$\lambda_{f^{(v_i, v_j)}} \ge	\rho \lambda_f$} 
				{
					$\mathcal{V}_k$ = $\mathcal{V}_k \cup \{v_j\}$;\\
					$\mathcal{E}_k$ = $\mathcal{E}_k \cup \{e_{v_i, v_j}\}$;\\
					$NS$ = $NS \cup \{v_j\}$;\\
				}
			}
		}
		$CS = NS$;\\
		$NS = \emptyset$;\\
		$m = m+1$;
	}
	\Return $\mathcal{G}_k$;
	\label{alg:metric_sampling}
\end{algorithm}
We utilize two previous proposed metrics as follows. At first, {\distance} $\lambda_d$ is used to determine the overall size of subgraphs. {\our} Sampler adopts a sampling strategy that expands from a random root node outwards sequentially, thus the number of expansion steps can control the scale of subgraphs to a certain extent. By setting the number of expansion steps to $\lfloor \lambda_d/2 \rfloor$, we can ensure that the hop distance between two nodes in the sampled subgraph does not constitute a failure connection, which can avoid much negative information. Here the value of $\lambda_d$ is estimated by the labeled nodes in the training set.

Second, {\feature} $\lambda_f$ is used to select neighboring nodes with high-quantity information gain during the expansion process. As we prove in Section~\ref{sec:featuresmooth}, the $\lambda_{f^{(v_i, v_j)}}$ can measure the information between two connected nodes. Here, we select the neighboring nodes with the condition $\lambda_{f^{(v_i, v_j)}} \ge \rho \lambda_f$, which means selecting those nodes with higher information gain than the overall level of the graph $\mathcal{G}$ but discard those with smaller information gain. Here, $\rho$ is a hyper-parameter to control the feature smoothness-based selection criteria. 

In Figure~\ref{fig:metric_sampling}, we illustrate the sampling process with the {\distance} of $6$, which enables a more intuitive understanding of each expansion. More details described by pseudocode are exhibited in Algorithm~\ref{alg:metric_sampling}.

\subsection{Subgraph-based Training in {\our}}
{\our} can support training most widely used GNN models (e.g., GAT, GCN) to avoid the three problems mentioned above. 
Different from training GNN models with the full graph $\mathcal{G}$, {\our} employs the subgraphs sampled from $\mathcal{G}$ by {\our} Sampler in each training iteration. In this way, a smaller size of coefficient matrices and only part of nodes are loaded into the GNN model during each training iteration. For a specific iteration, a subgraph $\mathcal{G}_t$ is selected from the set of sampled subgraphs $\mathbb{B}$, whose own the ground truth of nodes in $\mb{y}_{\mathcal{G}_t}$. The GNN model $H_{\mb{W}}(\cdot)$ to be trained is built on $\mathcal{G}_t$ and calculates the loss via forwarding propagation. Then, the weights $\mb{W}$ of the GNN model are updated via SGD. After enough iterations, we can achieve the GNN model with trained weights. We describe the training process of {\our} in detail by the pseudocode in Algorithm~\ref{alg:gnn_subgraph}.

\begin{algorithm}[t]
	\caption{{\our} Training for {\gnn}} 
	\KwIn{Graph $\mathcal{G}$; {\gnn} model $H_{\mb{W}}(\cdot)$; loss function $Loss(\cdot)$; subgraph mini-batch size $M$}
	\KwOut{Trained $H_{\mb{W}}(\cdot)$}
	Initialize training subgraph set $\mathbb{B} = \emptyset$;\\
	\For {$k = 1,2,\dots , M$}
	{
		${\mathcal{G}}_k \gets \textit{{\our} Sampler} $; \enspace/* By Algorithm~\ref{alg:metric_sampling} */ \\
		$\mathbb{B} = \mathbb{B} \cup \{\mathcal{G}_k\}$;
	}
	\For {each iteration}
	{
		Select a subgraph $\mathcal{G}_t$ from $\textit{batch}$;\\
		\vspace{0.02in}
		GNN model $H_{\mb{W}}(\cdot)$ construction on $\mathcal{G}_t$;\\
		Forward propagation to calculate the loss value: $loss = Loss( H_{\mb{W}}(\mathcal{G}_t), \mb{y}_{\mathcal{G}_t})$; \enspace /* $\mb{y}_{\mathcal{G}_t}$ denotes the ground truth of nodes in $\mathcal{G}_t$. */\\
		\vspace{0.02in}
		Backpropagation to update weights $\mb{W}$;
	}
	
	\Return $H_{\mb{W}}(\cdot)$
	\label{alg:gnn_subgraph}
	
\end{algorithm}

\subsection{Representation Aggregation-based Prediction}
After we achieve the GNN model with trained weights, these models can already make predictions for unlabeled nodes (testing samples). When facing the inductive learning in multiple graphs, we can feedforward the multiple unseen graphs in the test set to the GNN model directly. However, for the transductive learning in one single large graph, it is still difficult to feed the full graph into the model and implement the forward propagation to make predictions. Especially the memory space may not be able to load the full graph. In this case, {\our} continues to use the sampled subgraph to make predictions in a semi-supervised way.

There are two problems in using subgraphs to predict: 1, in the training subgraph set $\mathbb{B}$, may not all unlabeled nodes are included. 2, unlabeled nodes may be sampled into multiple subgraphs, so it is possible to output multiple conflicting prediction results. For problem 1, {\our} will sample extra subgraphs using the missing unlabeled nodes of $\mathbb{B}$ as root nodes. These extra subgraphs along with $\mathbb{B}$ will constitute the testing subgraph set $\mathbb{T}$. To deal with problem 2, {\our} implements aggregation on multiple representations of the same node and trains the predictor (e.g. classification layer) with the aggregated representations of labeled nodes. For example, a labeled node $v_i$ are sampled in $\mathcal{G}_{t_1},\mathcal{G}_{t_2},\mathcal{G}_{t_3}$, and we extract the representations  $\mb{h}_{v_i}^{\mathcal{G}_{t_1}}, \mb{h}_{v_i}^{\mathcal{G}_{t_2}}, \mb{h}_{v_i}^{\mathcal{G}_{t_3}}$ learned by the GNN model from each subgraph. In this paper, we use the mean aggregator to combine all representations as:
\begin{equation}
	\mb{h}_{v_i} \leftarrow \textbf{MEAN}(\{\mb{h}_{v_i}^{\mathcal{G}_{t_1}}, \mb{h}_{v_i}^{\mathcal{G}_{t_2}}, \mb{h}_{v_i}^{\mathcal{G}_{t_3}}\})
\end{equation}
The aggregated representation $\mb{h}_{v_i}$ will be used to train a new predictor. The aggregation process is the same for unlabeled nodes. The aggregated representations of unlabeled nodes will be used to make predictions by the new predictor. The representations of the same node in different subgraphs essentially embed the content of partial neighboring nodes and topology information of different local parts of the full graph. Therefore, the way to aggregate the representations of the same node in different subgraphs can enable the final representation with more comprehensive information.

\section{Experiments}\label{sec:experiment}

To show the effectiveness and efficiency of {\our}, extensive experiments have been conducted on benchmark datasets. This section first describes the datasets used in experiments and then introduces the experimental settings in detail. Generally, we aim to answer the following evaluation questions based on experimental results together with the detailed analysis:
\begin{itemize}
	\item \textbf{Question 1}: Can the subgraph-based training of {\our} effectively train GNN models and overcome the aforementioned three problems?
	\item \textbf{Question 2}: Can the {\our} Sampler provide powerful subgraphs to support effective and efficient training?
	\item \textbf{Question 3}: Can the representation aggregation-based prediction of {\our} improve the performance of original GNN models?
	
\end{itemize}
\subsection{Experiment Settings}

\subsubsection{Datasets}
\begin{table*}[t]
	\caption{Statistics of the Datasets in Experiments}
	\small
	\renewcommand\arraystretch{1.1}
	\centering
	\begin{threeparttable}
		\begin{tabular}{p{1.5cm} c c c c c}
			\hline
			\hline
			\multirow{2}*{}&\multicolumn{3}{c}{Transductive}&\multicolumn{2}{c}{Inductive}\\
			
			\cline{2-6}
			&Cora&Citeseer&Pubmed&Flickr&Reddit\\
			\hline
			\# Nodes&2708&3327&19717&89250&232965\\
			\# Edges&5429&4732&44338&899756&11606919\\
			\# Features&1433&3703&500&500&602\\
			\# classes&7&6&3&7&41\\
			Train Rate&0.052&0.036&0.003&0.5&0.66\\
			\hline
			\hline
		\end{tabular}
		
	\end{threeparttable}
	\label{tab:dataset}
\end{table*}

\begin{table*}[t]
	\caption{Test Accuracy Results on All Datasets}
	\renewcommand\arraystretch{1.2}
	\centering
	\begin{threeparttable}
		\begin{tabular}{l c c c c c}
			\hline
			\hline
			\multirow{2}*{\textbf{Methods}}&\multicolumn{3}{c}{Transductive}&\multicolumn{2}{c}{Inductive}\\
			
			\cline{2-6}
			&Cora&Citeseer&Pubmed&Flickr&Reddit\\
			
			\hline
			GraphSAGE&0.8096 $\pm$0.0103&0.6772$\pm$0.0110&0.7589$\pm$0.0161&0.4337$\pm$0.0201&0.9353$\pm$0.0211\\ 
			Cluster-GCN&0.6820$\pm$0.0638&0.6280$\pm$0.0430&0.7947$\pm$0.0036&0.4097$\pm$0.0405&\textbf{0.9523}$\pm$\textbf{0.0338}\\
			GraphSAINT&0.8045$\pm$0.0114&0.6961$\pm$0.0299&0.7424$\pm$0.0202&0.4848$\pm$0.0252&0.9451$\pm$0.0062\\
			\cline{1-6}
			GCN&0.8150$\pm$0.0050&0.7030$\pm$0.0050&0.7890$\pm$0.0070&0.4400$\pm$0.0388&0.9333$\pm$0.0147\\
			\hline
			GCN + Random&0.7945$\pm$0.0105&0.687$\pm$0.0133&0.7345$\pm$0.0163&0.4713$\pm$0.0083&0.8243$\pm$0.0336\\
			GCN + BFS&0.8144$\pm$0.0045&0.7079$\pm$0.0106&0.7971$\pm$0.0021&0.4754$\pm$0.0046&0.8123$\pm$0.0434\\
			GCN + {\ripplewalk}&0.8250$\pm$0.0016&0.7127$\pm$0.0018&0.8259$\pm$0.0225& 0.4797$\pm$0.0038&0.9385$\pm$0.0309\\
			{\our}$_{GCN}$&\textbf{0.8327}$\pm$\textbf{0.0130} &\textbf{0.7170}$\pm$\textbf{0.0252}&\textbf{0.8317}$\pm$\textbf{0.0167}&\textbf{0.5189}$\pm$\textbf{0.0121}&\textbf{0.9499}$\pm$\textbf{0.0232}\\
			\cline{1-6}
			GAT&\textbf{0.8300}$\pm$\textbf{0.0070}&0.7130$\pm$0.0082&0.7903$\pm$0.0033&-&-\\
			\hline
			GAT + Random&0.7921$\pm$0.0236&0.6607$\pm$0.0376&0.6765$\pm$0.0415&0.4534$\pm$0.0040&0.6452$\pm$0.0447\\
			GAT + BFS&0.7756$\pm$0.0281&0.6500$\pm$0.0493&0.7080$\pm$0.0511&0.4642$\pm$0.0232&0.7297$\pm$0.0403\\
			GAT + {\ripplewalk}&0.7994$\pm$0.0309&0.7212$\pm$0.0142&\textbf{0.8210}$\pm$\textbf{0.0019}&0.4724$\pm$0.0089&0.8699$\pm$0.0179\\
			{\our}$_{GAT}$&0.8017$\pm$0.0110&\textbf{0.7250}$\pm$\textbf{0.0121}&0.8174$\pm$0.0275&\textbf{0.4808}$\pm$\textbf{0.0301}&\textbf{0.8776}$\pm$\textbf{0.0320}\\
			\hline
			\hline
		\end{tabular}
		\begin{tablenotes}
			\footnotesize
			\item``-'' insufficient memory.
		\end{tablenotes}
	\end{threeparttable}
	\label{tab:results_all}
\end{table*}

Five benchmark datasets are used to evaluate the proposed framework {\our}: Cora, Citeseer, PubMed~\cite{sen2008collective}, Flickr, and Reddit. The descriptions of every dataset locate in Table~\ref{tab:dataset}. Cora, Citeseer, and Pubmed~\cite{sen2008collective} are standard citation network benchmark datasets. Flickr~\cite{zeng2019graphsaint,mcauley2012image} is built by forming links between images sharing common metadata from Flickr. Edges are formed between images from the same location, submitted to the same gallery, group, or set, images sharing common tags, images taken by friends, etc. Zeng et al.~\cite{zeng2019graphsaint} scan over the $81$ tags of each image and manually merged them to $7$ classes. Each image belongs to one of the $7$ classes, which is used as the label for images. Reddit~\cite{hamilton2017inductive} is a graph dataset constructed from Reddit posts. The node label is the community, or "subreddit" that a post belongs to. These datasets involve both the transductive learning task and inductive learning task. The transductive task in our experiments is semi-supervised node classification on one single graph; the inductive task is the node classification on multiple graphs. The statistic information of them is presented in Table~\ref{tab:dataset}. The train data rate in the table means the ratio of training data over the full dataset.

\subsubsection{GNN Models for Learning}
In our experiments, we adopt two powerful and widely used GNN models, GCN~\cite{kipf2016semi} and GAT~\cite{gat} as the basic models. The settings of these two models are different for various datasets. Specifically, both GCN and GAT contain two layers, and the size of the hidden layer is 32 for Cora, Citeseer, and Pubmed datasets; for Flickr and Reddit datasets, the GCN hidden layer size is 128, and the GAT hidden layer size is 32. For the training of GNNs, we select a part of node data as the validation set: for Cora, Citeseer, and Pubmed datasets, each of the validation sets includes 500 nodes, and the test set contains 1000 nodes; for Flickr and Reddit datasets, two percent of all nodes are included as validation sets, and another ten thousand nodes are selected into the test sets. We set the dropout rate as 0.5 and adopt the Adam~\cite{adam} as the optimizer for back-propagation. The learning rate is 0.01, and the weight decay rate is $5\times 10^{-4}$.

\begin{figure}[t]
	\centering
	\begin{minipage}[l]{\columnwidth}
		\centering
		\includegraphics[width=0.45\textwidth]{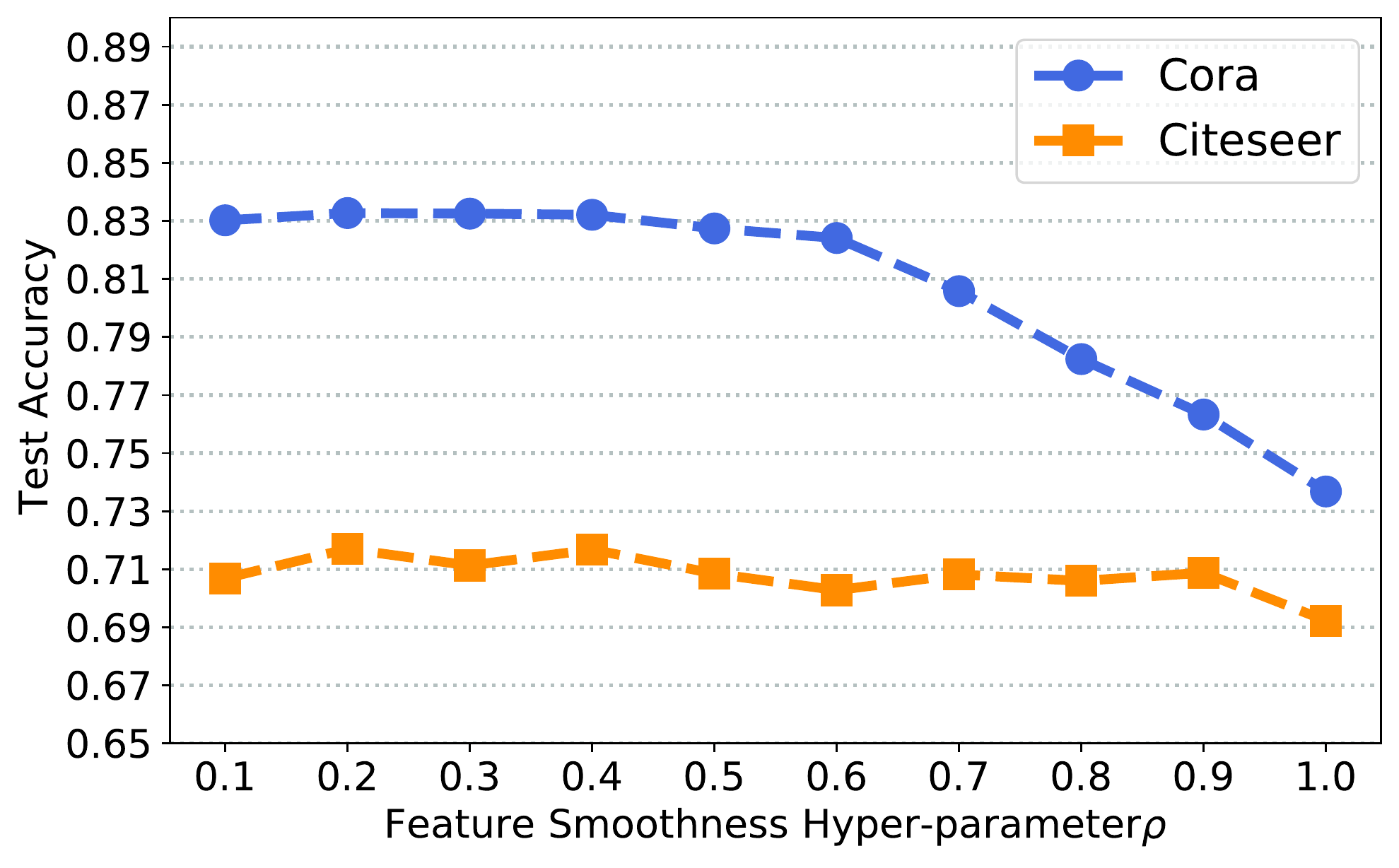}
	\end{minipage}
	\caption{{\our}$_{GCN}$ Results with Different Feature Smoothness Hyper-parameter $\rho$.}
	\label{fig:rho}
\end{figure}

\subsubsection{Comparison Methods}
We compare the proposed {\our} with several state-of-the-art baseline methods, including vanilla GCN, GAT, GraphSAGE, Cluster-GCN, GraphSAINT, and RippleWalk (RWT). These comparison methods can be divided into two categories: (1) full graph-based methods, which use the full graph in each training iteration; (2) subgraph-based methods involve only a subgraph in each train iteration. GCN and GAT belong to the first category for our selected baseline methods, while Cluster-GCN, GraphSAINT, and RWT are in the second category. A detailed description of the comparison methods is listed as follow:

\noindent\textbf{Full Graph-based Methods}
\begin{itemize}
	\item \textbf{GCN}~\cite{kipf2016semi}: GCN is a semi-supervised method proposed for the node classification task. The input of GCN is the full graph.
	\item \textbf{GAT}~\cite{gat}: GAT is an attention-based graph neural network for the node classification. GAT operates on the full graph.
	\item \textbf{GraphSAGE}~\cite{hamilton2017inductive}: GraphSAGE is a general inductive framework that leverages node feature information to efficiently generate node embeddings for previously unseen data. GraphSAGE does not require that all nodes in the graph are present during training.
\end{itemize}

\noindent\textbf{Subgraph-based Methods}
\begin{itemize}
	\item \textbf{\our}: {\our} is the general learning framework for GNN models proposed in this paper.
	\item \textbf{Cluster-GCN}~\cite{chiang2019cluster}: Cluster-GCN conducts subgraph sampling based on the graph clustering, and leverages the sampled subgraphs from different clusters to train the GNNs.
	\item \textbf{GraphSAINT}~\cite{zeng2019graphsaint}: GraphSAINT is designed for inductive learning on graph datasets. It samples subgraphs with different types of random samplers and uses these subgraphs to constitute the mini-batch to train GNN structures.
	\item \textbf{RippleWalk}~\cite{bai2020ripple}: RippleWalk (RWT) is a general training framework for GNN models on both transductive and inductive learning tasks. It samples subgraphs with the expansion steps from initial nodes, and return subgraphs with pre-defined target sizes. RWT constructs the mini-batch with sampled subgraphs to train GNN models.
	
\end{itemize}

Generally, the key of subgraph based methods is how to design the subgraph samplers. Apart from the subgraph-based methods mentioned above, we also compare {\our} with the following subgraph sampling methods:
\begin{itemize}
	\item \textbf{Random}: Random sampler generates subgraph by randomly select nodes from the full graph to a pre-defined target subgraph size.
	\item \textbf{BFS}: BFS sampler conducts the subgraph sampling from one initial node to expand with BFS strategy, and stop when the expanded subgraph reaches the target size. 
\end{itemize}

We replace {\our} Sampler with these two samplers in order to make comparisons. {\our}$_{GAT}$ denotes applying {\our} on the model GAT. GAT + BFS represents that we test {\our} on GAT with the setting of replacing {\our} Sampler with BFS. 

\subsection{Experiment Environment}
We run the experiments on the Server with 3 GTX-1080 ti GPUs. Codes are implemented in Pytorch 1.4.0, torch-geometric 1.6.0, cudatoolkit 10.1.243, and scikit-learn 0.23.1. Code is available at the github link: https://github.com/ \/ YuxiangRen/Measuring-and-Sampling-A-Metric-guided-Subgraph-Learning-Framework-for-Graph-Neural-Network.

\subsection{Experimental Results with Analysis}
\begin{figure}[t]
	\centering
	\subfigure[Comparison on GCN Learning]{\label{fig:Macro_F1}
		\begin{minipage}[l]{0.42\columnwidth}
			\centering
			\includegraphics[width=1\textwidth]{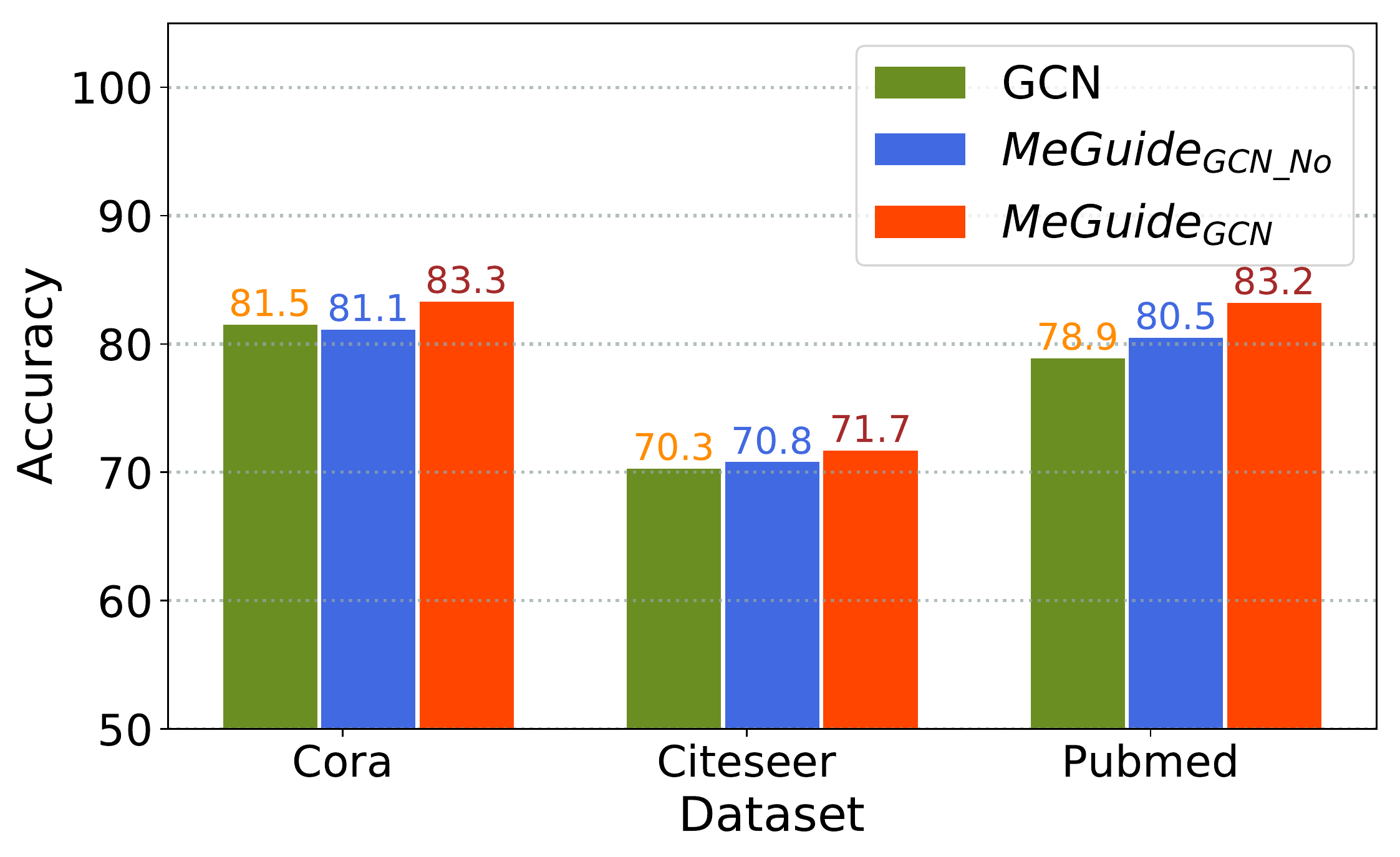}
		\end{minipage}
		\hspace{-10pt}
	}
	\subfigure[Comparison on GAT Learning]{\label{fig:Micro_F1}
		\begin{minipage}[l]{0.42\columnwidth}
			\centering
			\includegraphics[width=1\textwidth]{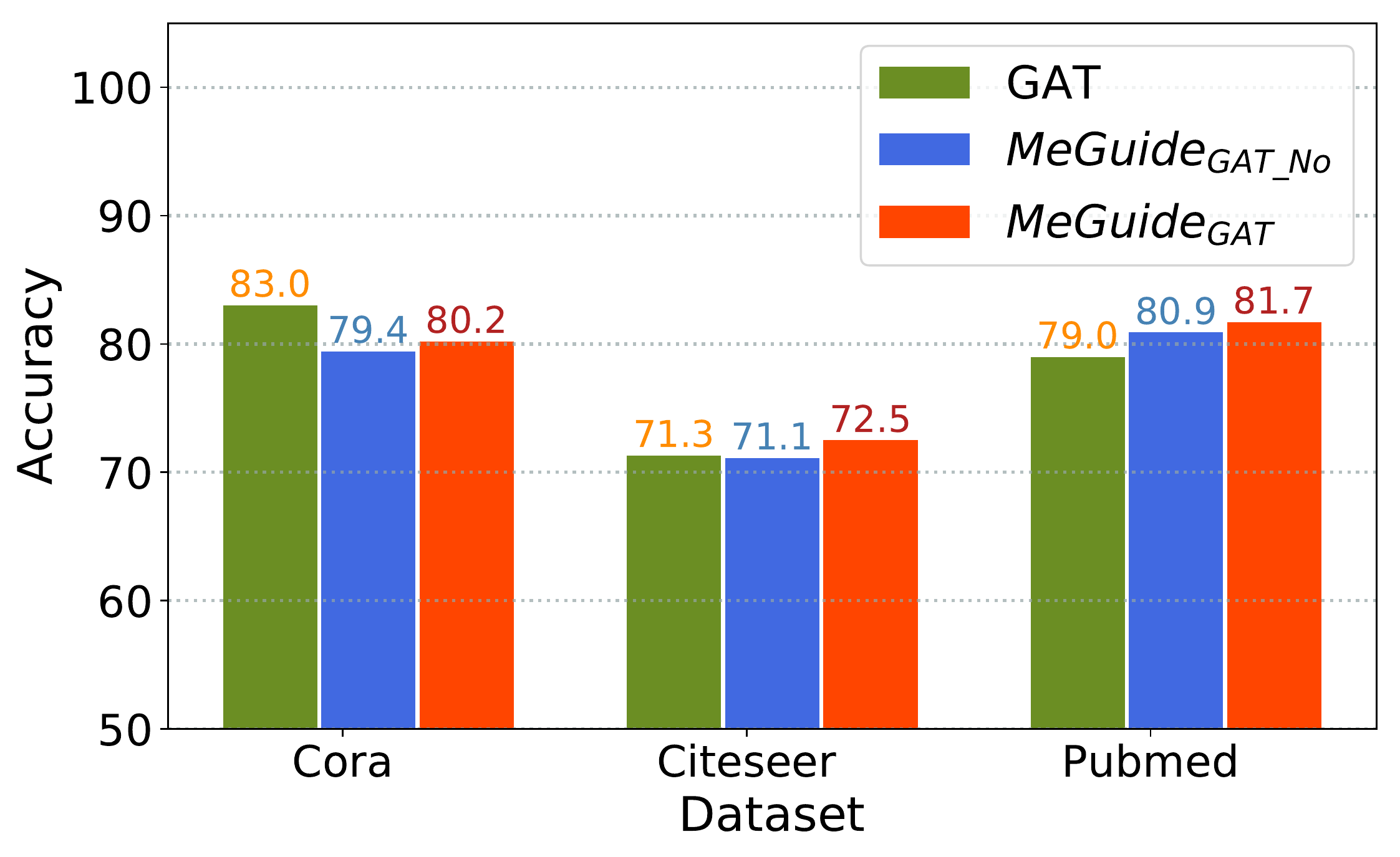}
		\end{minipage}
	}
	\caption{Ablation Study of Representation Aggregation-based Prediction.}\label{fig:ablation_study}
\end{figure}

\begin{figure*}[t]
	\centering
	\subfigure[GCN on Cora]{\label{fig:cora_memory}
		\begin{minipage}[l]{0.23\columnwidth}
			\centering
			\includegraphics[width=1\textwidth]{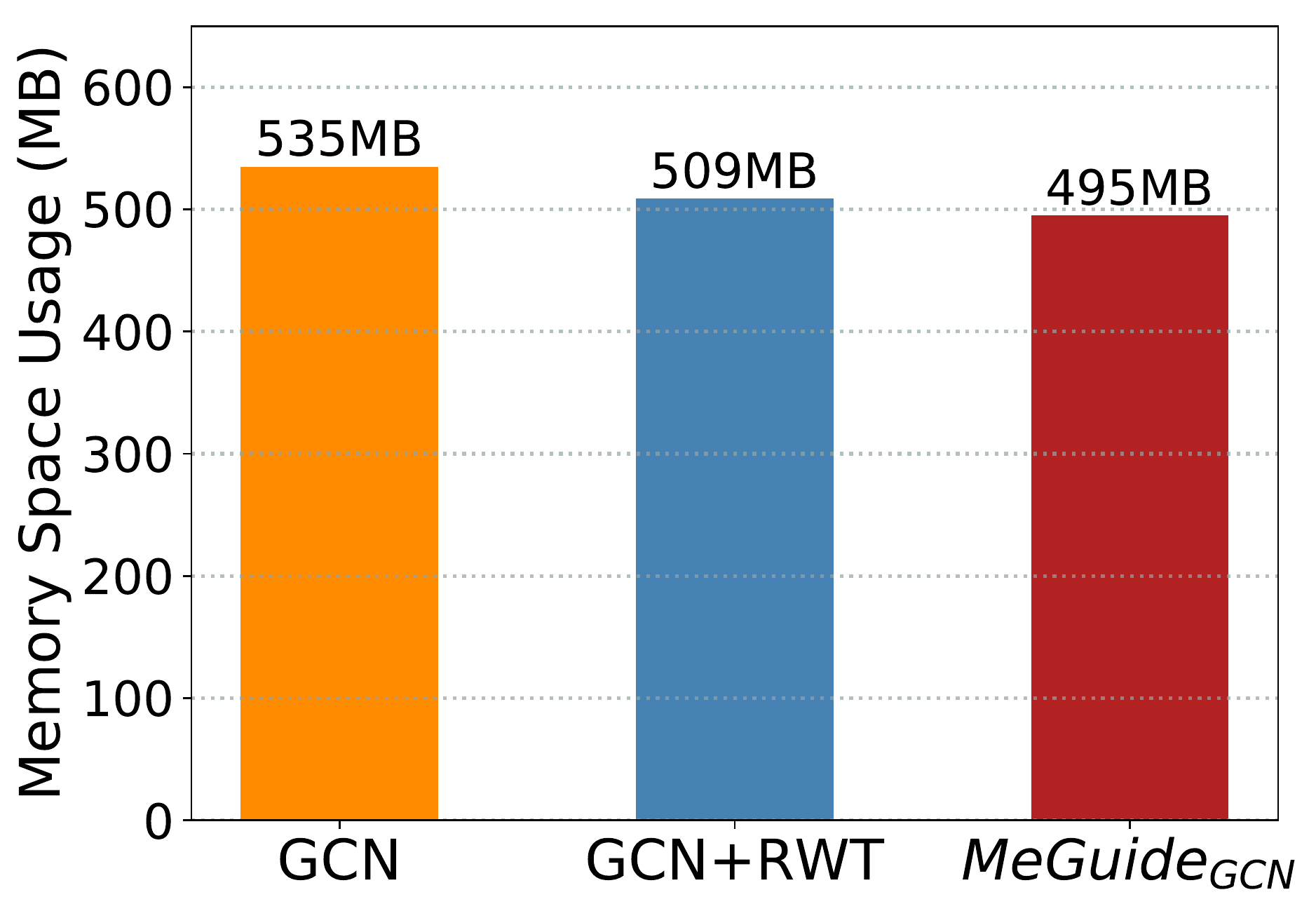}
		\end{minipage}
	}
	\subfigure[GCN on Citeseer]{\label{fig:citeseer_memory}
		\begin{minipage}[l]{0.23\columnwidth}
			\centering
			\includegraphics[width=1\textwidth]{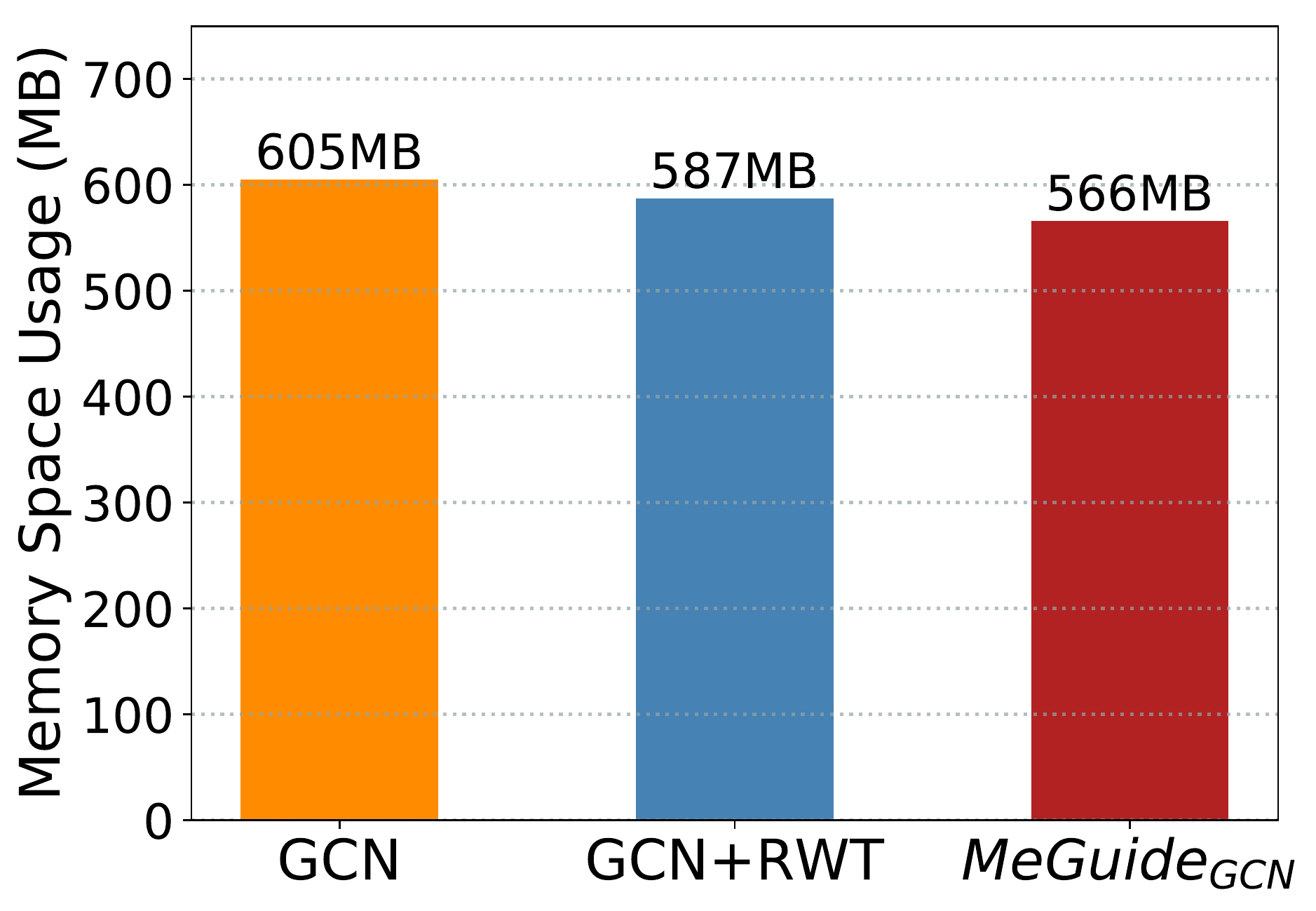}
		\end{minipage}
	}
	\subfigure[GCN on Pubmed]{\label{fig:pubmed_memory}
		\begin{minipage}[l]{0.23\columnwidth}
			\centering
			\includegraphics[width=1\textwidth]{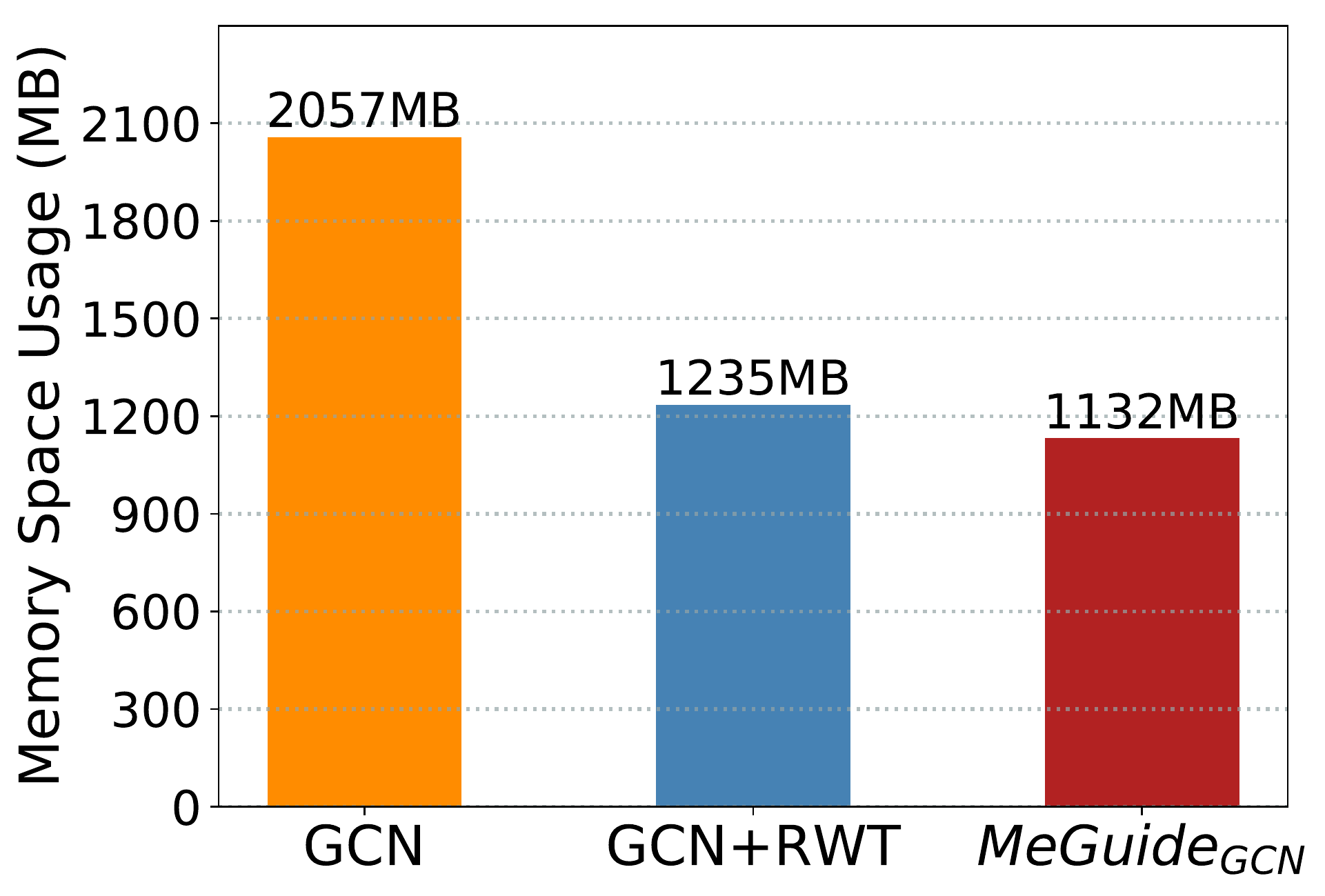}
		\end{minipage}
	}
	\subfigure[GCN on Flickr]{\label{fig:flickr_memory}
		\begin{minipage}[l]{0.23\columnwidth}
			\centering
			\includegraphics[width=1\textwidth]{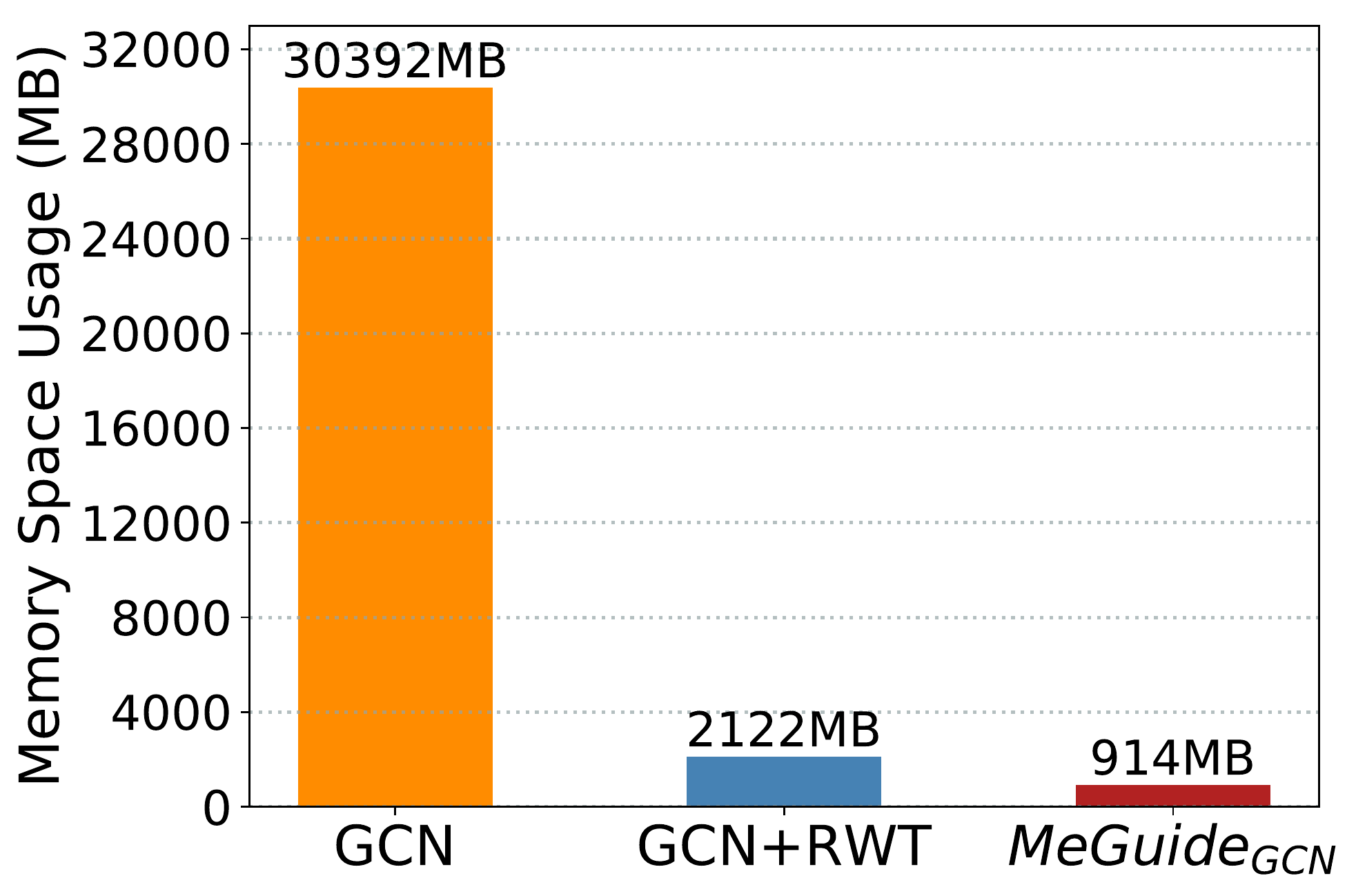}
		\end{minipage}
	}
	\subfigure[GCN on Reddit]{\label{fig:reddit_memory}
		\begin{minipage}[l]{0.23\columnwidth}
			\centering
			\includegraphics[width=1\textwidth]{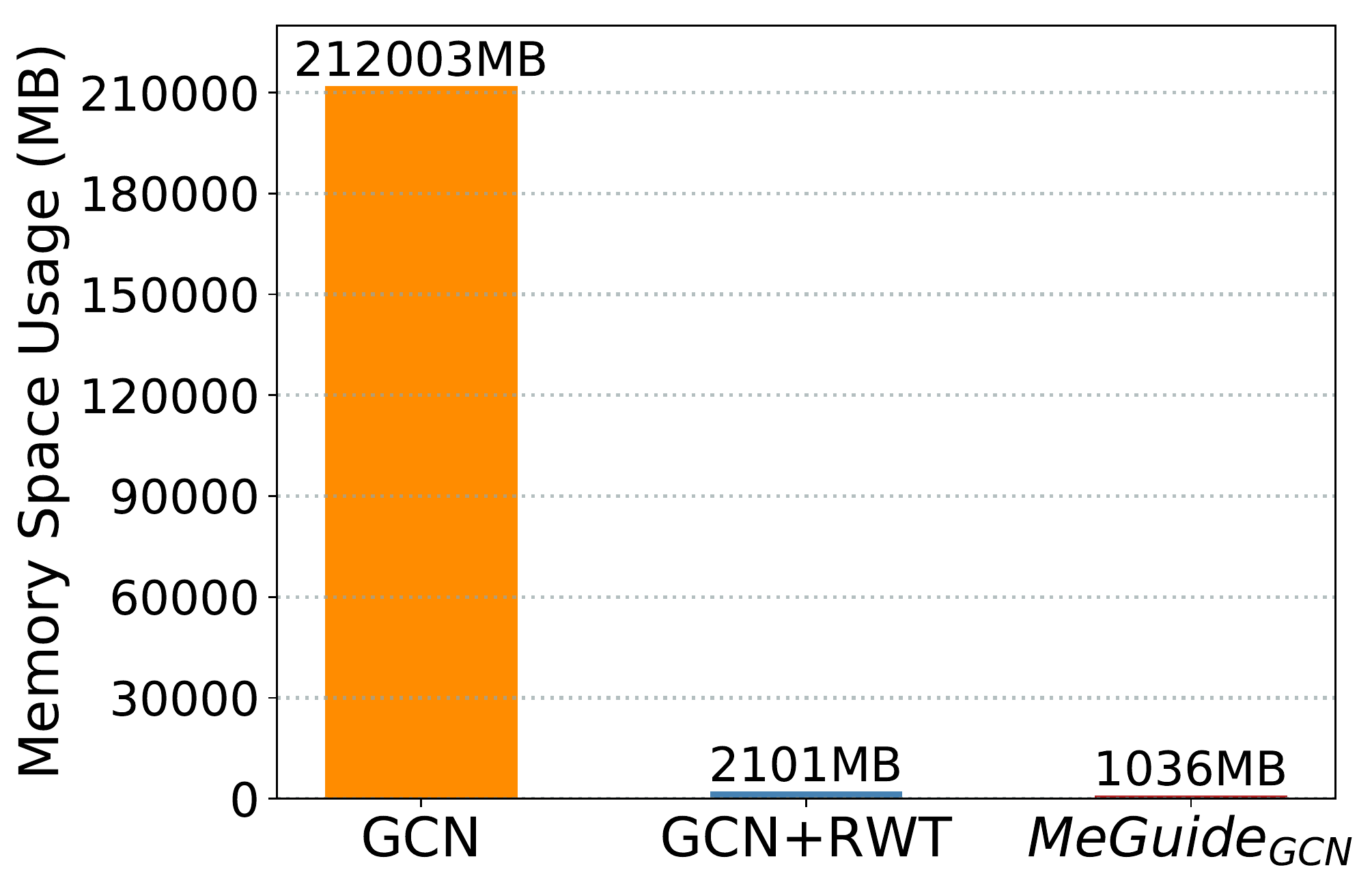}
		\end{minipage}
	}
	\subfigure[GAT on Cora]{\label{fig:cora_memory_GAT}
		\begin{minipage}[l]{0.23\columnwidth}
			\centering
			\includegraphics[width=1\textwidth]{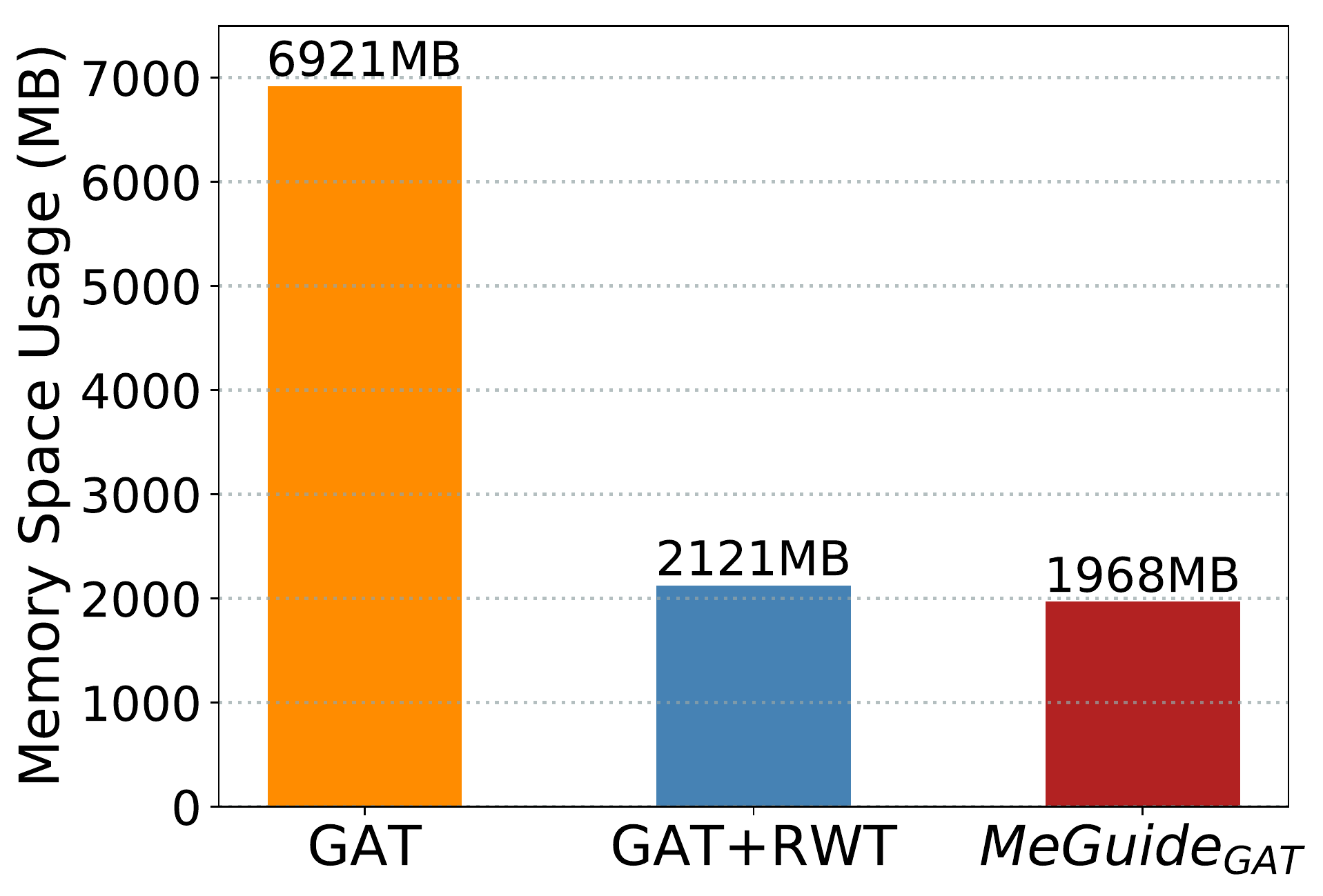}
		\end{minipage}
	}
	\subfigure[GAT on Citeseer]{\label{fig:citeseer_memory_GAT}
		\begin{minipage}[l]{0.23\columnwidth}
			\centering
			\includegraphics[width=1\textwidth]{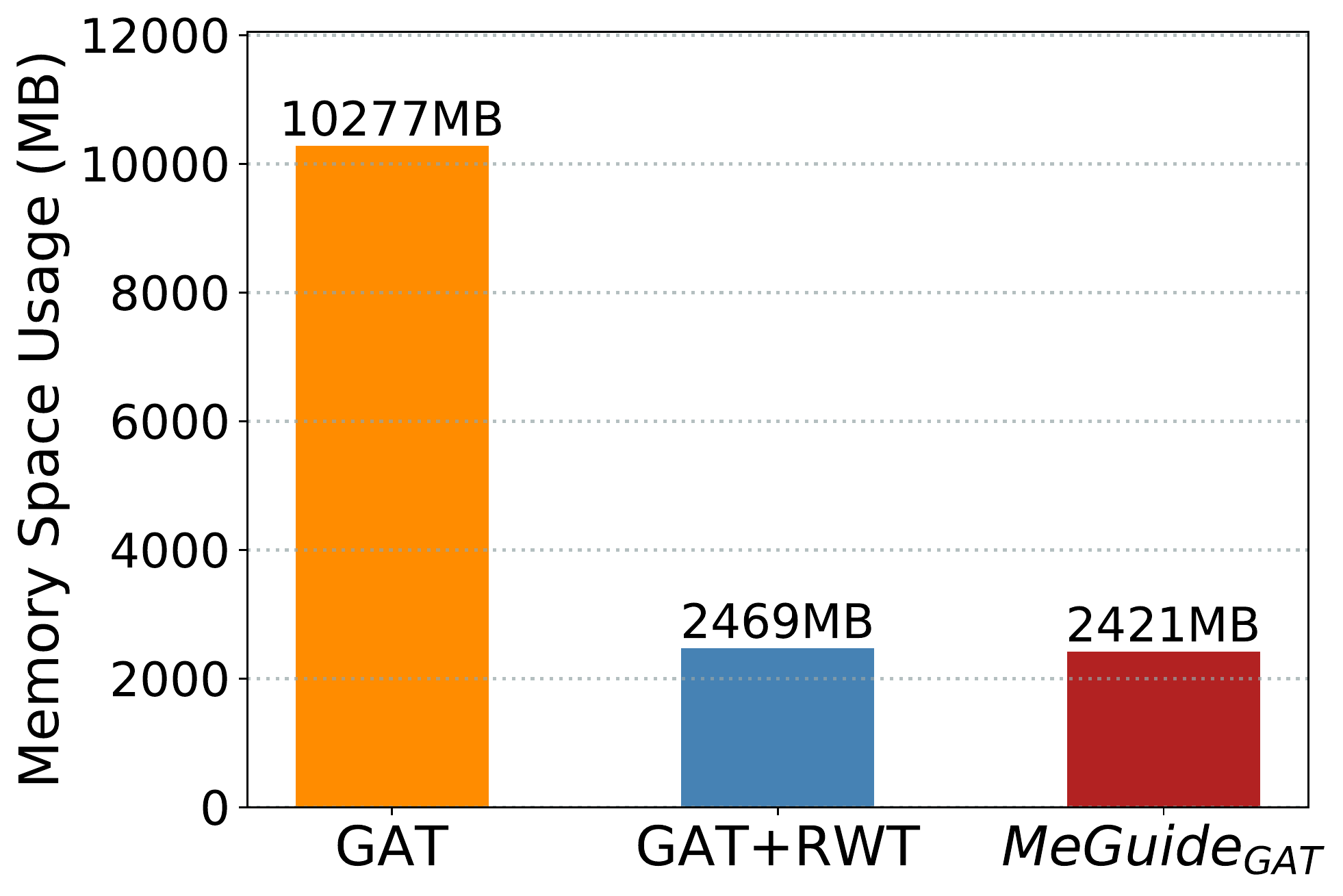}
		\end{minipage}
	}
	\subfigure[GAT on Pubmed]{\label{fig:pubmed_memory_GAT}
		\begin{minipage}[l]{0.23\columnwidth}
			\centering
			\includegraphics[width=1\textwidth]{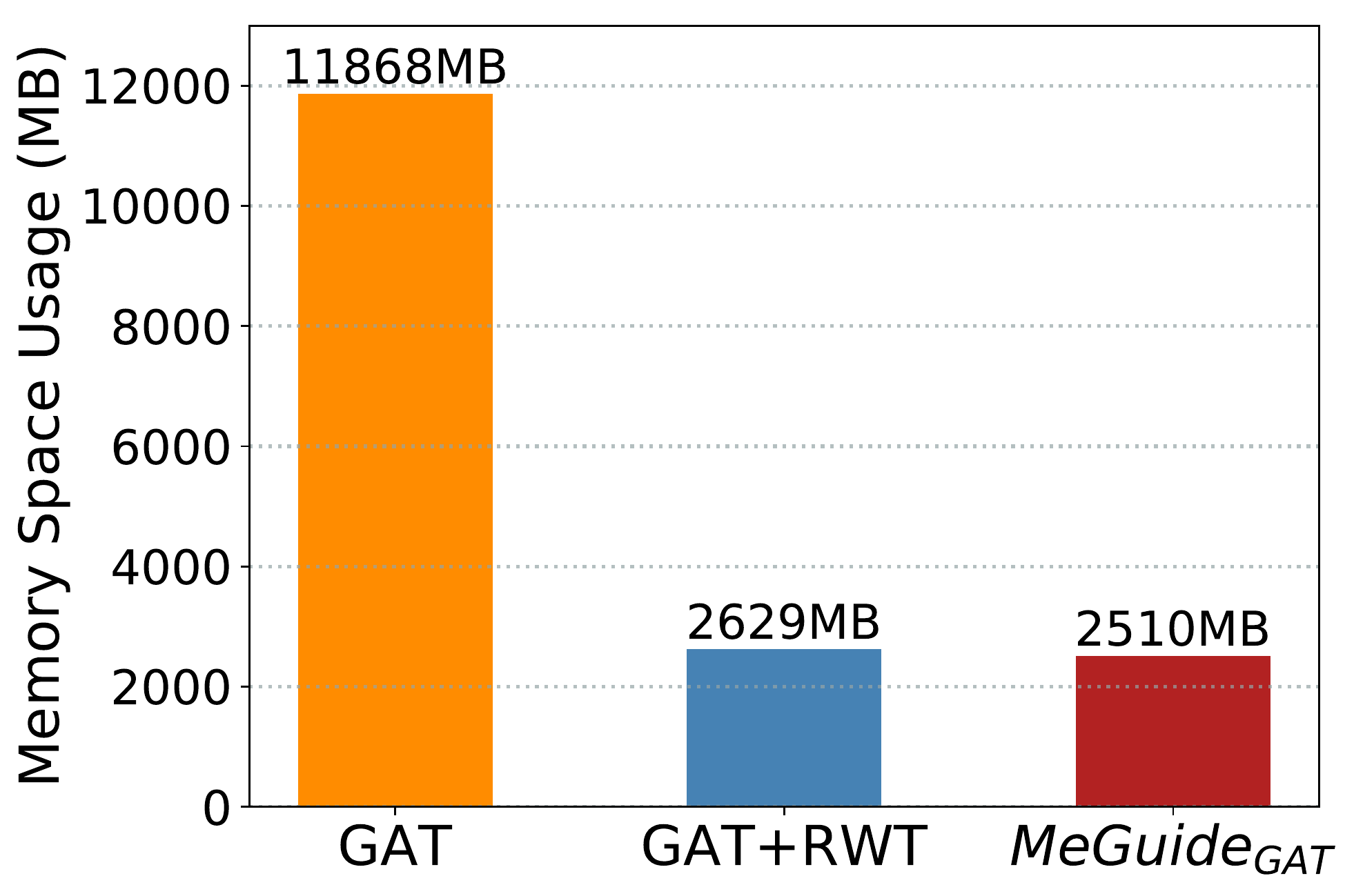}
		\end{minipage}
	}
	\subfigure[GAT on Flickr]{\label{fig:flickr_memory_GAT}
		\begin{minipage}[l]{0.23\columnwidth}
			\centering
			\includegraphics[width=1\textwidth]{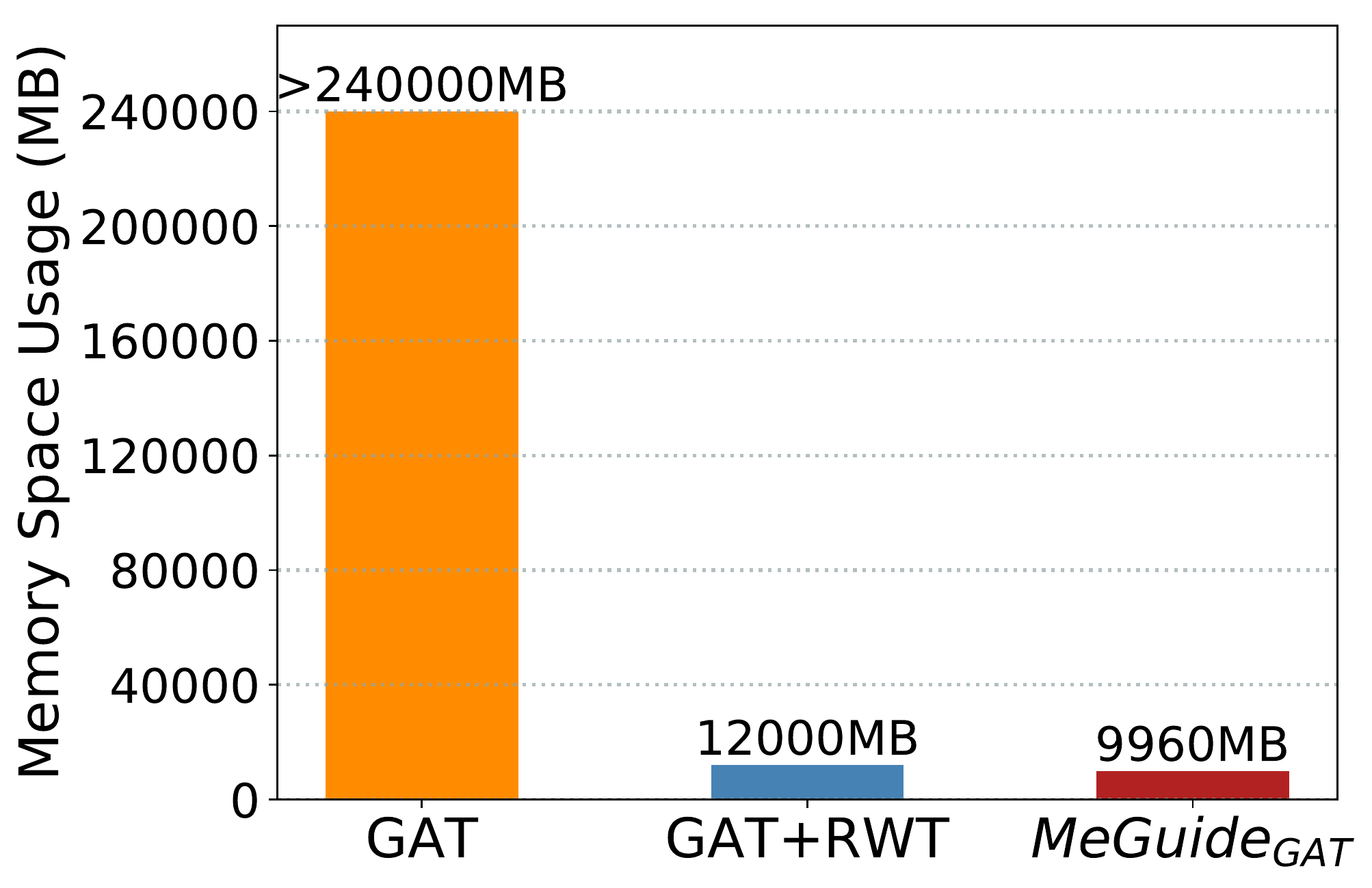}
		\end{minipage}
	}
	\subfigure[GAT on Reddit]{\label{fig:reddit_memory_GAT}
		\begin{minipage}[l]{0.23\columnwidth}
			\centering
			\includegraphics[width=1\textwidth]{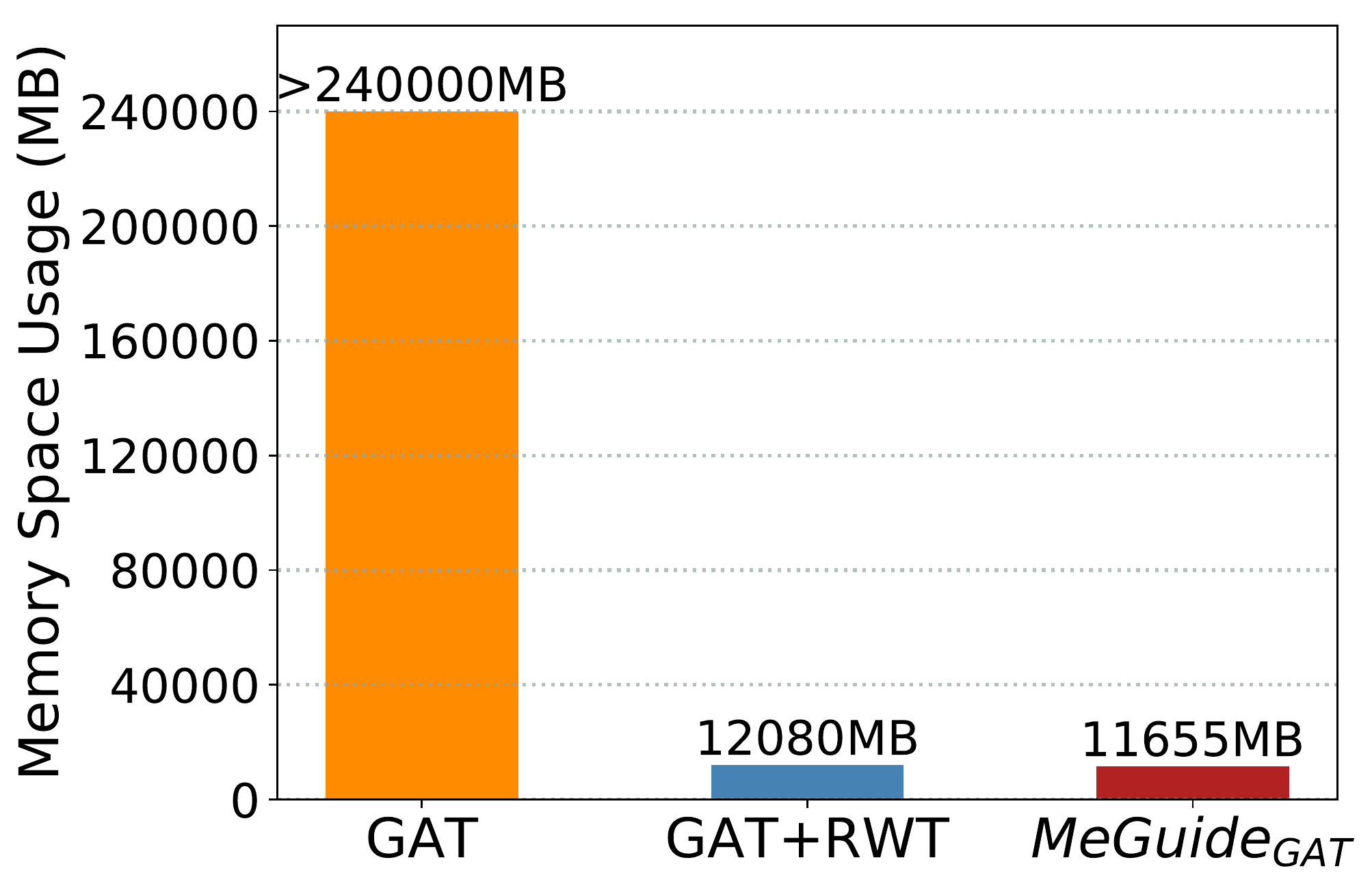}
		\end{minipage}
	}
	
	\caption{Memory Space Usage (The unit is MB) on Different Datasets.}\label{fig:memory}
\end{figure*}

\subsubsection{Overall performance analysis}
The results of both the transductive learning task and the inductive learning task are exhibited in Table~\ref{tab:results_all}. We present the accuracy of node classification on the test set. For the baseline methods evaluating the performance with F1 score in their original papers, out of the need for a consistent evaluation metric, we reimplement their models and run the experiments to report the results based on the metric accuracy. From Table~\ref{tab:results_all}, we can observe that {\our} achieves the best results on 4 out of the 5 datasets, for cases when GCN and GAT as the original GNN models. Regarding the dataset (Reddit) that {\our} does not obtain the highest accuracy, {\our}'s performance is still competitive compared to the state-of-the-art methods. Therefore, to the \textbf{Question 1} we can confidently answer that {\our} can guarantee the performance of the original GNN models while even help improves the performance of the GNNs. Apart from this, comparing to other subgraph-based methods, including Cluster-GCN, GraphSAINT, and RWT, {\our} acquires the best results on almost all datasets. This comparison also verifies the effectiveness in GNN model learning, which also illustrates the ability of {\our} to overcome the three problems encountered by GNN models from the perspective of the final result. In the following part, we will continue to discuss the performance of {\our} in dealing with the three problems separately.

To answer to the \textbf{Question 2}, we can compare the performance among different samplers.
The improvement from {\our} sampler can reach up to 10 percent compared to Random and BFS sampler (e.g., on Pubmed dataset) and 4 percent against RWT (e.g., on Flickr dataset). Such a phenomenon illustrates that the subgraphs involved by {\our} Sampler are powerful and be able to support effective training for GNN models.

To answer to the \textbf{Question 3}, we conducted a set of ablation studies. We remove the representation aggregation-based prediction in {\our} and directly input the full graph to the trained GNN models for prediction. The experimental results are shown in Figure~\ref{fig:ablation_study}, where ${\our}_{GCN\_No}$ represents that applying {\our} without the representation aggregation-based prediction to GCN. It can be seen from the comparison between GCN and ${\our}_{GCN\_No}$ that the prediction strategy based on the subgraph we designed can improve the performance of trained GNN models. The result is consistent for the comparison between GAT and ${\our}_{GAT\_No}$. Besides, the comparison between GCN and ${\our}_{GCN\_No}$ and the comparison between GAT and ${\our}_{GAT\_No}$ in the Figure~\ref{fig:ablation_study} can further demonstrate that {\our} can better train the original GNN models, and this performance improvement does not come from the representation aggregation-based prediction phase.

\subsubsection{Space-consuming Analysis}
In each experiment, we record the memory usage of our device and show the values in Figure~\ref{fig:memory}. From the figures, we can see that when using GCN as the original model, on the Cora and Citeseer dataset, the memory usage of {\our} is smaller than GCN and RWT by a little gap. GCN requires much more memory space than subgraph-based methods when it comes to relatively larger datasets (e.g., Pubmed, Flicker, and Reddit). Such a phenomenon is reasonable according to the mechanism of GNNs layer computation. Each layer of GNNs computing loads the current graph's adjacency matrix, the size of which will increase quadratically along with the increasing of graph size (i.e., the number of nodes). A similar phenomenon also occurs when GAT becomes the original model. Generally, compared to full graph-based methods, the advantage of {\our} is prominent concerning the memory cost. When comparing to RWT, {\our} keeps at the same level with RWT and has a slight advantage. Therefore, for the \textbf{Question 1}, {\our} successfully break through the memory bottleneck brought by \textit{node dependence} and \textit{neighbors explosion} (mentioned in Section~\ref{sec:intro}).

\subsubsection{Time-consuming Analysis}
In the training process, we record the convergence time of each experiment and list the results in Table~\ref{tab:time}. From the results, it is obvious that subgraph-based methods such as RWT and {\our} require much less convergence time than the full graph-based methods GCN and GAT. Such fast training convergence speed of subgraph-based methods is easy to understand. Since for the feedforward and backpropagation processes of GNNs computing, subgraph-based methods (e.g., {\our}) only involve subgraphs instead of the full graph, which limits the occurrence of \textit{node dependence} and \textit{neighbors explosion} to the scope of the subgraph.
Therefore, the computation complexity in each training iteration for subgraph-based methods is much lower than full graph-based methods. Besides, compared to RWT, the convergence of {\our} is also slightly faster, but {\our} can achieve better performance, which reflects the effectiveness and efficiency of the subgraph sampling by {\our}.

\begin{table*}[h]
	\caption{Convergence Time (The unit is second)}
	\renewcommand\arraystretch{1.1}
	\centering
	\begin{tabular}{l c c c c c}
		\hline
		\hline
		&Cora&Citeseer&Pubmed&Flickr&Reddit\\
		\hline
		{\gcn}&4.573&1.968&61.90&1161.92&25370\\
		{\gcn} + RWT&1.964&1.826&8.698&1.179&7.722\\
		{\our}$_{GCN}$&\textbf{1.784}&\textbf{1.653}&\textbf{5.112}&\textbf{1.149}&\textbf{7.262}\\
		\hline
		{\gat}&413.3&500.1&-&-&-\\
		{\gat} + RWT&71.44&47.06&139.4&68.06&2614\\
		{\our}$_{GAT}$&\textbf{63.75}&\textbf{36.88}&\textbf{109.7}&\textbf{58.30}&\textbf{2367}\\
		\hline
		\hline
	\end{tabular}
	\begin{tablenotes}
		\footnotesize
		\item``-'' insufficient memory for GPU.
	\end{tablenotes}
	\label{tab:time}
\end{table*}

\subsection{Parameter Analysis}
\begin{table*}[h]
	\caption{Metrics Value on All Datasets}
	\small
	\renewcommand\arraystretch{1.1}
	\centering
	\begin{tabular}{l  c  c  c  c  c}
		\hline
		\hline
		&Cora&Citeseer&Pubmed&Flickr&Reddit\\
		\hline
		$\lambda_f$&0.123&0.051&0.058&1002.042&775.54\\
		$\lambda_d$&8.8&6.5&6.9&5.8&3.9\\
		\hline
		\hline
	\end{tabular}
	\label{tab:metric}
\end{table*}
There are some key parameters and metric values employed by our proposed {\our}. We list the feature smoothness values $\lambda_f$ and connection failure distance $\lambda_d$ of all datasets in Table~\ref{tab:metric}. On Cora, Citeseer, and Pubmed datasets, the feature smoothness value is relatively small, illustrating that the nodes in the whole graph share more similar features. The subgraph sampling strategy of {\our} is based on the values in the table. There is another important hyper-parameter, the feature smoothness hyper-parameter $\rho$, which determines the expansion condition of {\our} sampler. We tune different values of $\rho$ and exhibit the results in Figure~\ref{fig:rho}. The trend of results indicates that {\our} achieves the highest performance when $\rho$ locates in the range of 0.2 to 0.5. When $\rho$ is too small (e.g., $\rho = 0.1$) or becomes larger, the performance worsens. Such a phenomenon matches our expectation, because on the one hand, when $\rho$ is too small, {\our} is not able to filter nodes with less information gain during the sampling process of {\our} sampler. On the other hand, when $\rho$ is too large, the expansion step of {\our} sampler cannot work anymore: only a few neighboring nodes of the already selected nodes are qualified for being sampled.

\section{Conclusion}
\label{sec:conclusion}
In this paper, we propose a general framework {\our} for optimizing the training and prediction of GNN models. In {\our}, we design the subgraph-based training to deal with three non-trivial problems (\textit{neighbors explosion}, \textit{node dependence}, and \textit{oversmoothing}) bothering many GNN models. Different from the existing subgraph-based training methods, we define two metrics that can be used to measure the performance gain of GNN models to guide the sampling of subgraphs. The training effect of the GNN models can be improved through more effective subgraphs. In addition, for the case of the memory bottleneck when using trained GNN models to predict on a single large graph, {\our} provides a solution for prediction based on subgraphs, which is an unsolved problem left by existing subgraph-based training methods. Extensive experiments on 5 benchmark graph datasets and 2 widely used GNN models demonstrate the effectiveness of {\our}, where GNN models not only achieve the comparative or even better performance but less training time and device memory space is required. 

%
%
%

\bibliography{refs}

\begin{thebibliography}{45}
\providecommand{\natexlab}[1]{#1}
\providecommand{\url}[1]{\texttt{#1}}
\providecommand{\urlprefix}{}

\bibitem[{Zhang(2018)Zhang, Jiawei}]{JZhang19}
Zhang J.
\newblock Social network fusion and mining: a survey.
\newblock arXiv preprint arXiv:180409874 2018;.

\bibitem[{Ren and Zhang(2020)Ren, Yuxiang and Zhang, Jiawei}]{ren2020hgat}
Ren Y, Zhang J.
\newblock HGAT: Hierarchical Graph Attention Network for Fake News Detection.
\newblock arXiv 2020;p. arXiv--2002.

\bibitem[{Wang et~al.(2017)Wang, Quan and Mao, Zhendong and Wang, Bin and Guo,
  Li}]{WMWG17}
Wang Q, Mao Z, Wang B, Guo L.
\newblock Knowledge graph embedding: A survey of approaches and applications.
\newblock IEEE Transactions on Knowledge and Data Engineering
  2017;29(12):2724--2743.

\bibitem[{Bai et~al.(2021)Bai, Jiyang and Ren, Yuxiang and Zhang,
  Jiawei}]{bai2020ripple}
Bai J, Ren Y, Zhang J.
\newblock Ripple Walk Training: A Subgraph-based training framework for Large
  and Deep Graph Neural Network.
\newblock In: IJCNN; 2021. .

\bibitem[{Xu et~al.(2019)Xu, Keyulu and Hu, Weihua and Leskovec, Jure and
  Jegelka, Stefanie}]{XHLJ10}
Xu K, Hu W, Leskovec J, Jegelka S.
\newblock How powerful are graph neural networks?
\newblock In: ICLR; 2019. .

\bibitem[{Chiang et~al.(2019)Chiang, Wei-Lin and Liu, Xuanqing and Si, Si and
  Li, Yang and Bengio, Samy and Hsieh, Cho-Jui}]{chiang2019cluster}
Chiang WL, Liu X, Si S, Li Y, Bengio S, Hsieh CJ.
\newblock Cluster-GCN: An Efficient Algorithm for Training Deep and Large Graph
  Convolutional Networks.
\newblock In: KDD; 2019. .

\bibitem[{Zhao and Akoglu(2019)Zhao, Lingxiao and Akoglu,
  Leman}]{zhao2019pairnorm}
Zhao L, Akoglu L.
\newblock PairNorm: Tackling Oversmoothing in GNNs.
\newblock arXiv:190912223 2019;.

\bibitem[{Li et~al.(2018)Li, Qimai and Han, Zhichao and Wu,
  Xiao-Ming}]{li2018deeper}
Li Q, Han Z, Wu XM.
\newblock Deeper insights into graph convolutional networks for semi-supervised
  learning.
\newblock In: AAAI; 2018. .

\bibitem[{Rong et~al.(2019)Rong, Yu and Huang, Wenbing and Xu, Tingyang and
  Huang, Junzhou}]{rong2019dropedge}
Rong Y, Huang W, Xu T, Huang J.
\newblock DropEdge: Towards the Very Deep Graph Convolutional Networks for Node
  Classification.
\newblock In: arXiv:1907.10903; 2019. .

\bibitem[{Hamilton et~al.(2017)Hamilton, Will and Ying, Zhitao and Leskovec,
  Jure}]{hamilton2017inductive}
Hamilton W, Ying Z, Leskovec J.
\newblock Inductive representation learning on large graphs.
\newblock In: NIPS; 2017. .

\bibitem[{Chen et~al.(2018{\natexlab{a}})Chen, Jie and Ma, Tengfei and Xiao,
  Cao}]{chen2018fastgcn}
Chen J, Ma T, Xiao C.
\newblock Fastgcn: fast learning with graph convolutional networks via
  importance sampling.
\newblock arXiv:180110247 2018;.

\bibitem[{Chen et~al.(2018{\natexlab{b}})Chen, Jianfei and Zhu, Jun and Song,
  Le}]{chen2017stochastic}
Chen J, Zhu J, Song L.
\newblock Stochastic training of graph convolutional networks with variance
  reduction.
\newblock In: ICML; 2018. .

\bibitem[{Zou et~al.(2019)Zou, Difan and Hu, Ziniu and Wang, Yewen and Jiang,
  Song and Sun, Yizhou and Gu, Quanquan}]{zou2019layer}
Zou D, Hu Z, Wang Y, Jiang S, Sun Y, Gu Q.
\newblock Layer-Dependent Importance Sampling for Training Deep and Large Graph
  Convolutional Networks.
\newblock In: NeurIPS; 2019. .

\bibitem[{Zeng et~al.(2019)Zeng, Hanqing and Zhou, Hongkuan and Srivastava,
  Ajitesh and Kannan, Rajgopal and Prasanna, Viktor}]{zeng2019graphsaint}
Zeng H, Zhou H, Srivastava A, Kannan R, Prasanna V.
\newblock Graphsaint: Graph sampling based inductive learning method.
\newblock arXiv preprint arXiv:190704931 2019;.

\bibitem[{Hou et~al.(2019)Hou, Yifan and Zhang, Jian and Cheng, James and Ma,
  Kaili and Ma, Richard TB and Chen, Hongzhi and Yang,
  Ming-Chang}]{hou2019measuring}
Hou Y, Zhang J, Cheng J, Ma K, Ma RT, Chen H, et~al.
\newblock Measuring and improving the use of graph information in graph neural
  networks.
\newblock In: International Conference on Learning Representations; 2019. .

\bibitem[{Bruna et~al.(2013)Bruna, Joan and Zaremba, Wojciech and Szlam, Arthur
  and LeCun, Yann}]{bruna2013spectral}
Bruna J, Zaremba W, Szlam A, LeCun Y.
\newblock Spectral networks and locally connected networks on graphs.
\newblock arXiv:13126203 2013;.

\bibitem[{Veli{\v{c}}kovi{\'{c}} et~al.(2018)Petar Veli{\v{c}}kovi{\'{c}} and
  Guillem Cucurull and Arantxa Casanova and Adriana Romero and Pietro Li{\`{o}}
  and Yoshua Bengio}]{gat}
Veli{\v{c}}kovi{\'{c}} P, Cucurull G, Casanova A, Romero A, Li{\`{o}} P, Bengio
  Y.
\newblock Graph Attention Networks.
\newblock In: ICLR; 2018. .

\bibitem[{Monti et~al.(2017)Monti, Federico and Boscaini, Davide and Masci,
  Jonathan and Rodola, Emanuele and Svoboda, Jan and Bronstein, Michael
  M}]{monti2017geometric}
Monti F, Boscaini D, Masci J, Rodola E, Svoboda J, Bronstein MM.
\newblock Geometric deep learning on graphs and manifolds using mixture model
  cnns.
\newblock In: CVPR; 2017. .

\bibitem[{Kipf and Welling(2017)Kipf, Thomas N and Welling, Max}]{kipf2016semi}
Kipf TN, Welling M.
\newblock Semi-supervised classification with graph convolutional networks.
\newblock In: ICLR; 2017. .

\bibitem[{Xinyi and Chen(2018)Xinyi, Zhang and Chen, Lihui}]{xinyi2018capsule}
Xinyi Z, Chen L.
\newblock Capsule graph neural network.
\newblock In: International conference on learning representations; 2018. .

\bibitem[{Sun et~al.(2019)Sun, Fan-Yun and Hoffmann, Jordan and Verma, Vikas
  and Tang, Jian}]{sun2019infograph}
Sun FY, Hoffmann J, Verma V, Tang J.
\newblock Infograph: Unsupervised and semi-supervised graph-level
  representation learning via mutual information maximization.
\newblock arXiv preprint arXiv:190801000 2019;.

\bibitem[{Ying et~al.(2018)Ying, Zhitao and You, Jiaxuan and Morris,
  Christopher and Ren, Xiang and Hamilton, Will and Leskovec,
  Jure}]{ying2018hierarchical}
Ying Z, You J, Morris C, Ren X, Hamilton W, Leskovec J.
\newblock Hierarchical graph representation learning with differentiable
  pooling.
\newblock In: Advances in neural information processing systems; 2018. p.
  4800--4810.

\bibitem[{Ren et~al.(2021)Ren, Yuxiang and Bai, Jiyang and Zhang,
  Jiawei}]{ren2021label}
Ren Y, Bai J, Zhang J.
\newblock Label Contrastive Coding based Graph Neural Network for Graph
  Classification.
\newblock In: Database Systems for Advanced Applications Springer International
  Publishing; 2021. p. 123--140.

\bibitem[{Zhou et~al.(2020)Zhou, Jie and Cui, Ganqu and Hu, Shengding and
  Zhang, Zhengyan and Yang, Cheng and Liu, Zhiyuan and Wang, Lifeng and Li,
  Changcheng and Sun, Maosong}]{zhou2020graph}
Zhou J, Cui G, Hu S, Zhang Z, Yang C, Liu Z, et~al.
\newblock Graph neural networks: A review of methods and applications.
\newblock AI Open 2020;1:57--81.

\bibitem[{Wu et~al.(2020)Wu, Zonghan and Pan, Shirui and Chen, Fengwen and
  Long, Guodong and Zhang, Chengqi and Philip, S Yu}]{wu2020comprehensive}
Wu Z, Pan S, Chen F, Long G, Zhang C, Philip SY.
\newblock A comprehensive survey on graph neural networks.
\newblock IEEE transactions on neural networks and learning systems
  2020;32(1):4--24.

\bibitem[{Defferrard et~al.(2016)Defferrard, Micha{\"e}l and Bresson, Xavier
  and Vandergheynst, Pierre}]{defferrard2016convolutional}
Defferrard M, Bresson X, Vandergheynst P.
\newblock Convolutional neural networks on graphs with fast localized spectral
  filtering.
\newblock In: NIPS; 2016. .

\bibitem[{Levie et~al.(2018)Levie, Ron and Monti, Federico and Bresson, Xavier
  and Bronstein, Michael M}]{levie2018cayleynets}
Levie R, Monti F, Bresson X, Bronstein MM.
\newblock Cayleynets: Graph convolutional neural networks with complex rational
  spectral filters.
\newblock IEEE Transactions on Signal Processing 2018;.

\bibitem[{Liao et~al.(2019)Liao, Renjie and Zhao, Zhizhen and Urtasun, Raquel
  and Zemel, Richard S}]{liao2019lanczosnet}
Liao R, Zhao Z, Urtasun R, Zemel RS.
\newblock Lanczosnet: Multi-scale deep graph convolutional networks.
\newblock arXiv:190101484 2019;.

\bibitem[{Henaff et~al.(2015)Henaff, Mikael and Bruna, Joan and LeCun,
  Yann}]{henaff2015deep}
Henaff M, Bruna J, LeCun Y.
\newblock Deep convolutional networks on graph-structured data.
\newblock arXiv:150605163 2015;.

\bibitem[{Li et~al.(2018)Li, Ruoyu and Wang, Sheng and Zhu, Feiyun and Huang,
  Junzhou}]{li2018adaptive}
Li R, Wang S, Zhu F, Huang J.
\newblock Adaptive graph convolutional neural networks.
\newblock In: AAAI; 2018. .

\bibitem[{Ying et~al.(2018)Ying, Rex and He, Ruining and Chen, Kaifeng and
  Eksombatchai, Pong and Hamilton, William L and Leskovec,
  Jure}]{ying2018graph}
Ying R, He R, Chen K, Eksombatchai P, Hamilton WL, Leskovec J.
\newblock Graph convolutional neural networks for web-scale recommender
  systems.
\newblock In: KDD; 2018. .

\bibitem[{Gao et~al.(2018)Gao, Hongyang and Wang, Zhengyang and Ji,
  Shuiwang}]{gao2018large}
Gao H, Wang Z, Ji S.
\newblock Large-scale learnable graph convolutional networks.
\newblock In: KDD; 2018. .

\bibitem[{Xu et~al.(2018)Xu, Keyulu and Li, Chengtao and Tian, Yonglong and
  Sonobe, Tomohiro and Kawarabayashi, Ken-ichi and Jegelka,
  Stefanie}]{xu2018representation}
Xu K, Li C, Tian Y, Sonobe T, Kawarabayashi Ki, Jegelka S.
\newblock Representation Learning on Graphs with Jumping Knowledge Networks.
\newblock In: ICML; 2018. .

\bibitem[{Lee et~al.(2019)Lee, John Boaz and Rossi, Ryan A and Kong, Xiangnan
  and Kim, Sungchul and Koh, Eunyee and Rao, Anup}]{lee2019graph}
Lee JB, Rossi RA, Kong X, Kim S, Koh E, Rao A.
\newblock Graph Convolutional Networks with Motif-based Attention.
\newblock In: CIKM; 2019. .

\bibitem[{Klicpera et~al.(2019)Klicpera, Johannes and Bojchevski, Aleksandar
  and G{\"u}nnemann, Stephan}]{klicpera2019predict}
Klicpera J, Bojchevski A, G{\"u}nnemann S.
\newblock Predict then Propagate: Graph Neural Networks meet Personalized
  PageRank 2019;.

\bibitem[{Haonan et~al.(2019)Haonan, Lu and Huang, Seth H and Ye, Tian and
  Xiuyan, Guo}]{haonan2019graph}
Haonan L, Huang SH, Ye T, Xiuyan G.
\newblock Graph star net for generalized multi-task learning.
\newblock arXiv:190612330 2019;.

\bibitem[{Abu-El-Haija et~al.(2019)Abu-El-Haija, Sami and Perozzi, Bryan and
  Kapoor, Amol}]{abu2019mixhop}
Abu-El-Haija S, Perozzi B, Kapoor A.
\newblock Mixhop: Higher-order graph convolution architectures via sparsified
  neighborhood mixing.
\newblock arXiv:190500067 2019;.

\bibitem[{Chen et~al.(2020)Chen, Ming and Wei, Zhewei and Ding, Bolin and Li,
  Yaliang and Yuan, Ye and Du, Xiaoyong and Wen, Ji-Rong}]{chen2020scalable}
Chen M, Wei Z, Ding B, Li Y, Yuan Y, Du X, et~al.
\newblock Scalable graph neural networks via bidirectional propagation.
\newblock arXiv preprint arXiv:201015421 2020;.

\bibitem[{Zhang and Meng(2019)Zhang, Jiawei and Meng, Lin}]{gresnet}
Zhang J, Meng L.
\newblock GResNet: Graph Residual Network for Reviving Deep GNNs from Suspended
  Animation.
\newblock In: arXiv:1909.05729; 2019. .

\bibitem[{Kullback and Leibler(1951)Kullback, Solomon and Leibler, Richard
  A}]{kullback1951information}
Kullback S, Leibler RA.
\newblock On information and sufficiency.
\newblock The annals of mathematical statistics 1951;22(1):79--86.

\bibitem[{Zhang and Meng(2019)Zhang, Jiawei and Meng, Lin}]{zhang2019gresnet}
Zhang J, Meng L.
\newblock GResNet: Graph Residual Network for Reviving Deep GNNs from Suspended
  Animation.
\newblock arXiv preprint arXiv:190905729 2019;.

\bibitem[{Huang et~al.(2020)Huang, Wenbing and Rong, Yu and Xu, Tingyang and
  Sun, Fuchun and Huang, Junzhou}]{huang2020tackling}
Huang W, Rong Y, Xu T, Sun F, Huang J.
\newblock Tackling Over-Smoothing for General Graph Convolutional Networks.
\newblock arXiv preprint arXiv:200809864 2020;.

\bibitem[{Sen et~al.(2008)Sen, Prithviraj and Namata, Galileo and Bilgic,
  Mustafa and Getoor, Lise and Galligher, Brian and Eliassi-Rad,
  Tina}]{sen2008collective}
Sen P, Namata G, Bilgic M, Getoor L, Galligher B, Eliassi-Rad T.
\newblock Collective classification in network data.
\newblock AI magazine 2008;.

\bibitem[{McAuley and Leskovec(2012)McAuley, Julian and Leskovec,
  Jure}]{mcauley2012image}
McAuley J, Leskovec J.
\newblock Image labeling on a network: using social-network metadata for image
  classification.
\newblock In: European conference on computer vision Springer; 2012. p.
  828--841.

\bibitem[{Kingma and Ba(2015)Diederik P. Kingma and Jimmy Lei Ba}]{adam}
Kingma DP, Ba JL.
\newblock ADAM: A method for stochastic optimization.
\newblock In: ICLR; 2015. .

\end{thebibliography}

\begin{biography}[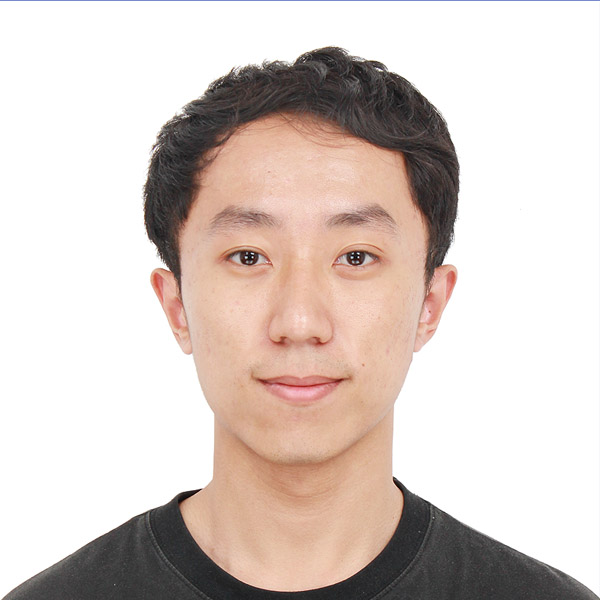]{Jiyang Bai}
	received the bachelor degree in information and numerical science from Nankai University, China, in 2018. He is pursuing a PhD degree in the Department of Computer Science at the Florida State University. His main research areas are data mining and machine learning, especially focus on the graph mining, graph neural networks, graph similarity search.    \\
\end{biography}

\begin{biography}[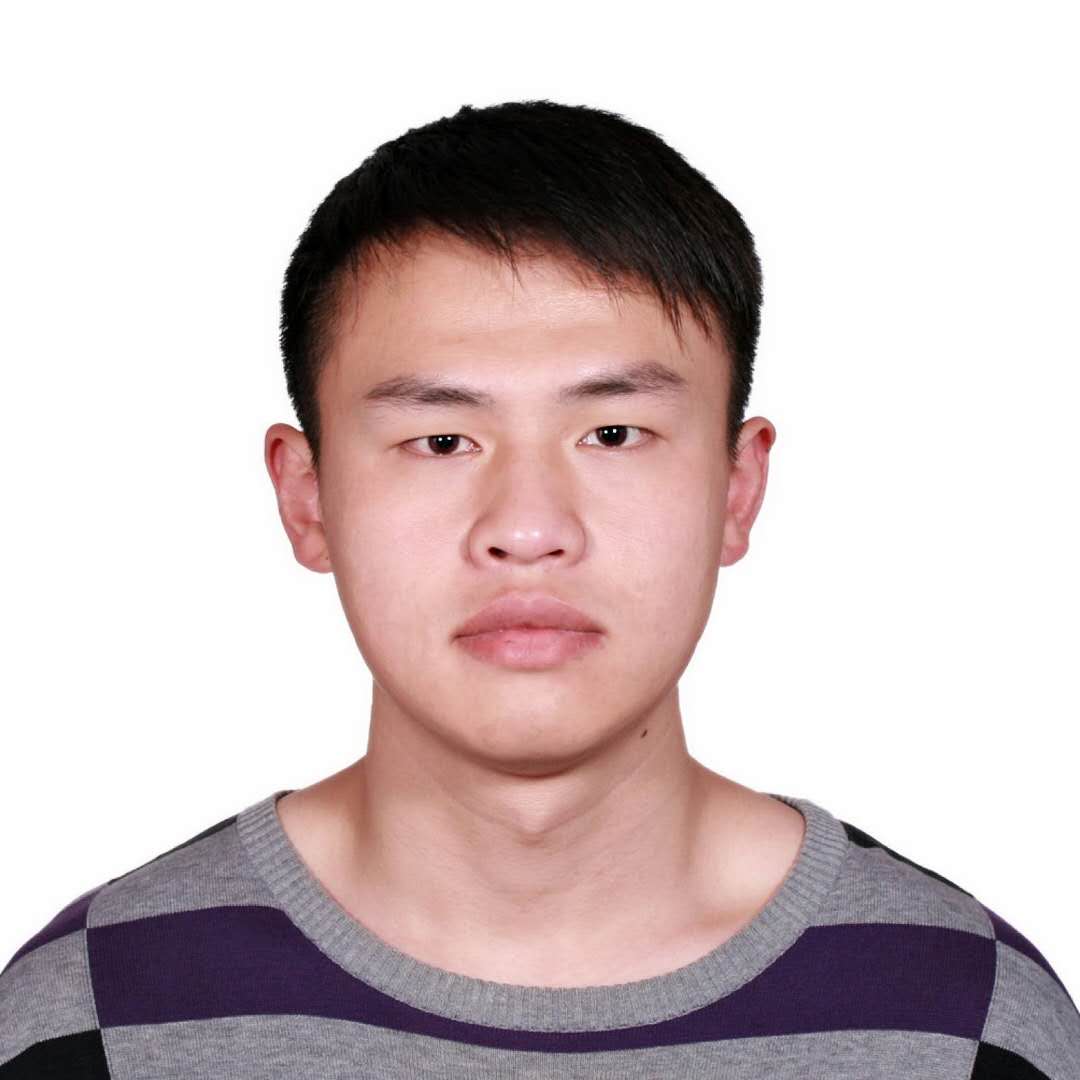]{Yuxiang Ren}
received the bachelor degree in
software engineering and the bachelor degree in law from Nanjing University, China,
in 2015, and the Ph.D. degree in Computer Science from Florida State University in 2021. His main research areas are data
mining and machine learning, especially focus on the development and analysis of algorithms for social and information networks, as well as heterogeneous graph mining and fake news detection. 
\end{biography}

\begin{biography}[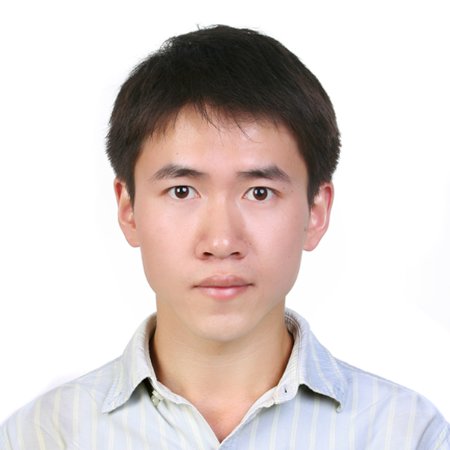]{Jiawei Zhang}
	received the bachelor’s degree in
	computer science from Nanjing University, China,
	in 2012, and the Ph.D. degree in computer science
	from the University of Illinois at Chicago, USA, in
	2017. He has been an Assistant Professor with the
	Department of Computer Science, Florida State
	University, Tallahassee, FL, USA, since 2017.
	He founded IFM Lab in 2017, and has been working as the director since then. IFM Lab is a research oriented academic laboratory, providing the latest information on fusion learning and data mining research works and application tools to both academia and industry.
\end{biography}

\end{document}